\documentclass[compsoc]{IEEEtran}

\pdfoutput=1 

\title{Observability Properties of Colored Graphs}
\author{Mark Chilenski, George Cybenko~\IEEEmembership{Fellow,~IEEE,} Isaac Dekine, Piyush Kumar, and Gil Raz\IEEEcompsocitemizethanks{\IEEEcompsocthanksitem M.~Chilenski, I.~Dekine, P.~Kumar, and G.~Raz are with Systems \& Technology Research. \IEEEcompsocthanksitem G.~Cybenko is with Dartmouth College. \IEEEcompsocthanksitem Accepted version of Chilenski et al.\ ``Observability properties of colored graphs,'' \emph{IEEE Transactions on Network Science and Engineering}, 2019. DOI: \protect\url{https://doi.org/10.1109/TNSE.2019.2948474} \IEEEcompsocthanksitem \textcopyright\ 2019 IEEE. Personal use of this material is permitted. Permission from IEEE must be obtained for all other uses, in any current or future media, including reprinting/republishing this material for advertising or promotional purposes, creating new collective works, for resale or redistribution to servers or lists, or reuse of any copyrighted component of this work in other works.}}

\usepackage[normalem]{ulem}

\usepackage{url}
\usepackage{xcolor}
\usepackage{amsmath}
\usepackage{amsfonts}
\usepackage{graphicx}
\usepackage{subfig}
\usepackage{xspace}
\usepackage{booktabs}
\usepackage[detect-all]{siunitx}
\usepackage{amsthm}

\newcommand{\bluecircle}{\raisebox{-1pt}{\includegraphics{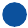}}\xspace}
\newcommand{\redsquare}{\raisebox{-1pt}{\includegraphics{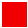}}\xspace}
\newcommand{\greendiamond}{\raisebox{-2.1pt}{\includegraphics{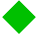}}\xspace}
\newcommand{\orangetriangle}{\raisebox{-1pt}{\includegraphics{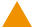}}\xspace}
\newcommand{\greyhex}{\raisebox{-1pt}{\includegraphics{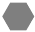}}\xspace}

\graphicspath{{./figures/}}

\newcommand{\ap}{\emph{a posteriori}\xspace}

\newtheorem{theorem}{Theorem}

\begin{document}
\maketitle

\begin{abstract}
	A colored graph is a directed graph in which nodes or edges have been assigned colors that are not necessarily unique. Observability problems in such graphs consider whether an agent observing the colors of edges or nodes traversed on a path in the graph can determine which node they are at currently or which nodes were visited earlier in the traversal. Previous research efforts have identified several different notions of observability as well as the associated properties of graphs for which those observability properties hold. This paper unifies the prior work into a common framework with several new results about relationships between those notions and associated graph properties. The new framework provides an intuitive way to reason about the attainable accuracy as a function of lag and time spent observing, and identifies simple modifications to improve the observability of a given graph. We show that one form of the graph modification problem is in NP-Complete. The intuition of the new framework is borne out with numerical experiments. This work has implications for problems that can be described in terms of an agent traversing a colored graph, including the reconstruction of hidden states in a hidden Markov model (HMM).
\end{abstract}


\section{Introduction}
\IEEEPARstart{C}{onsider} an agent traversing a directed graph whose nodes (or edges) are assigned one or more colors.
The agent seeks to localize itself within the graph based on the colors observed.
In such problems, the nodes can represent states, the edges 
can represent allowed state transitions, and the colors can represent the discrete symbols that can be emitted by a given state (in the case of node-colored graphs) or a given state transition (in the case of edge-colored graphs).  
Such models are used in a variety of applications, some of which are described in the following section.

We consider various formulations of inferring the nodes and edges visited by an agent when only the sequence of colors emitted can be observed directly.
What can be inferred about nodes and edges is collectively called an observability property.
Previous work has identified several classes of colored graphs, each having different implications for the tractability and accuracy of this inference problem \cite{ShengIEEE2005,CrespiACM2008,JungersDAM2011}.
Some of the related notions which have arisen in other contexts including symbolic dynamics \cite{Beal2011} and discrete event systems \cite{OzverenIEEE1990,SabooriIEEE2007,BryansIJIS2008,SabooriThesis2010} are nicely summarized in \cite{JungersDAM2011}.
The present work unites these concepts in a comprehensive framework based on the presence or absence of particular pathologies in the graph and its coloring.
We also provide a unified structure to reason about the effects of specific pathologies and to design mitigations that improve the observability properties of a given graph.

The readers should note that we are not studying the more challenging
problem of inferring the graph and its coloring from sequences of color observations which is a different computational problem
\cite{kearns1994cryptographic}.  In particular, we are assuming that the graph and its coloring are known \emph{a priori} and not being learned by the observer.

The rest of the paper is structured as follows: Section~\ref{sec:bg} describes our notation and reviews the previous work on colored graph models and observability classes.
Section~\ref{sec:pathos} presents the colored graph pathologies and discusses their implications and possible mitigations.
Section~\ref{sec:taxon} presents the relationships between colored graph observability classes.
Section~\ref{sec:examples} illustrates the implications of the various pathologies using simulations.
Finally, Section~\ref{sec:conc} summarizes the contributions of this paper and discusses potential directions for future work.

\section{Background and Previous Work}
\label{sec:bg}

\subsection{Definitions: Colored Graphs, Weak Models, and Hidden Markov Models}
\label{sec:bg:def}
A \emph{node-colored directed graph} $G = (V, E, L, \Phi)$ consists of a set of nodes $V$, a set of edges $E$ consisting of ordered pairs of nodes, a set of possible colors $\Phi$, and a mapping $L:V\to2^\Phi$ which indicates which subset of $\Phi$ can be emitted by a given node.
This is identical to the definition of a \emph{weak model} given in \cite{CrespiACM2008}.
(Note that some authors also include the set of nodes which the system can start at as part of their definition of a weak model \cite{JiangSPIE2004,ShengIEEE2005}.)

If each edge $(i,j)\in E$ is endowed with a transition probability $P_{ij}=P(X_{t+1}=j|X_{t}=i)$ (where $X_t$ is the node visited at time $t$) and each node $i\in V$ is endowed with a set of emission probabilities $B_{i\alpha}=P(Y_t=\alpha|X_t=i)$ (where $Y_t$ is the color emitted at time $t$ and $\alpha\in\Phi$), the model is a \emph{hidden Markov model} (HMM) with discrete symbols \cite{RabinerIEEE1989}.  Note that we are concerned here with structural properties of such systems that depend only on whether certain state transitions and emissions are possible or not, and not on the specific probabilities.
Consequently, our results are about whether certain inferences about colored graphs are true or not true, as opposed to what the probabilities or likelihoods of inferences are.

A node-colored graph is \emph{multi-colored} if there exists a $v\in V$ such that $|L(v)|>1$.
In the context of tracking an agent traversing the graph, the implication is that \emph{one} of the possible colors $c\in L(v)$ will be emitted when the agent visits node $v$.
It is often useful to reduce such a graph to the equivalent single-colored graph.
This can be accomplished by replacing every multi-colored node with multiple nodes, one for each color, then duplicating the appropriate edges.
This is illustrated in Figure~\ref{fig:mcTrans}.

\begin{figure}
	\centering
	\subfloat[Multi-colored]{\includegraphics[scale=0.15]{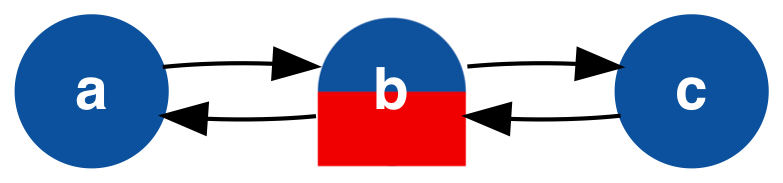}}
	\subfloat[Single-colored]{\includegraphics[scale=0.15]{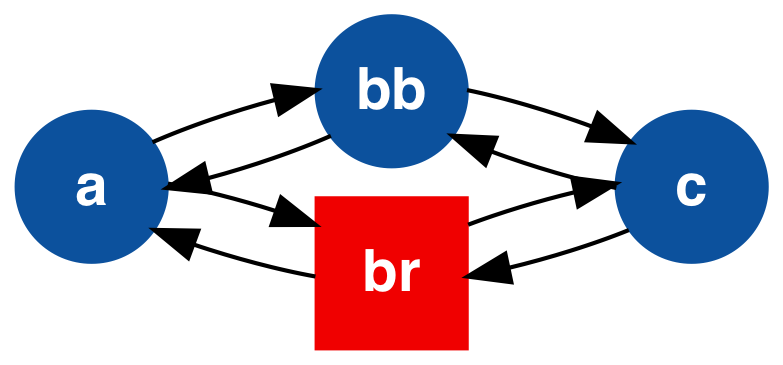}}
	\caption{Transformation of multi-colored graph to single-colored graph. Node $b$ can emit either \protect\bluecircle or \protect\redsquare, so it is split into nodes $bb$ and $br$, respectively. (This paper will use both color and shape to distinguish node ``colors.'')}
	\label{fig:mcTrans}
\end{figure}

An \emph{edge-colored directed graph} is defined as above, but instead associates the mapping to colors with the edges: $L:E\to 2^\Phi$.
As with the node-colored case, it is possible to have an \emph{edge-multi-colored graph} if there exists an $e\in E$ such that $|L(e)|>1$.
Edge-colored graphs can describe higher-order dependencies (e.g., they represent a system where the color emitted depends on both the current node and the previous node), so it is often useful to reduce an edge-colored graph to an equivalent node-colored graph.
This can be accomplished by replacing every node which has incident edges of more than one color by multiple nodes, one for each color, then assigning each node the color of its incident edges.
This is illustrated in Figure~\ref{fig:edgeTrans}.

Therefore, every node-multi-colored graph and every edge-colored graph, whether multi-colored or not, can be reduced to an equivalent node-colored graph for which each node emits only one color.  Moreover, the reduction results in modest growth of the graph.  Specifically, if there are $n$ nodes in the graph and the multi-colored node or edge with the most colors has $C$ colors, the resulting simply colored node graph will have no more than $Cn$ unicolored nodes.

Consequently, in the remainder of this paper, we consider, without loss of generality, node colored graphs for which each node can emit only one color.

\begin{figure}
	\centering
	\subfloat[Edge-multi-colored]{\includegraphics[scale=0.15]{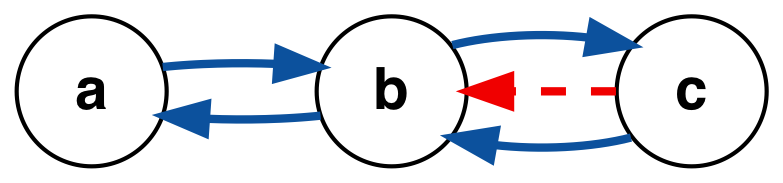}}
	\subfloat[Node-colored]{\includegraphics[scale=0.15]{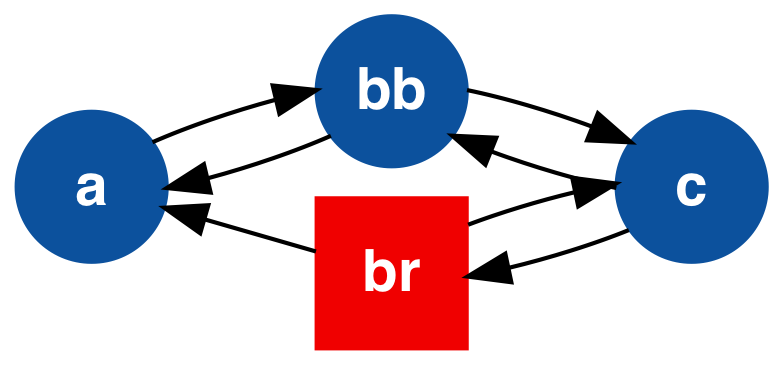}}
	\caption{Transformation of edge-multi-colored graph to node-colored graph. Node $b$ has incident edges which are both blue/solid and red/dashed, so it is split into blue/circle node $bb$ and red/square node $br$.}
	\label{fig:edgeTrans}
\end{figure}

\subsection{Example Applications}
\label{sec:apps}
Such models have been applied in a variety of contexts.
In our previous work in the cybersecurity domain, the nodes are the basic blocks of a software 
program's control flow graph (CFG), the edges are the allowed control flow transitions (dictated by jump, call, and return instructions in the program), and the colors are the specific signals emitted in a side-channel (e.g., electromagnetic emissions) \cite{ChilenskiSPIE2018,cybenko2018large}.
The colored graph model is used to reason about how accurately the program execution can be tracked in order to verify that the software is functioning as expected.

In a different cybersecurity context, colored graph models have been used as part of process query systems to detect attacks on computer networks \cite{BerkSPIE2005,ShengThesis2006,CybenkoC2007}.
Here, the nodes are the states of the attacker as they compromise various elements of the network and the colors are the signals produced by a variety of sensors used to instrument the network as part of its intrusion detection system.
In this context it is of interest to determine both which process (or processes, in the case of multiple simultaneous attacks) is running as well as the current state which any given process is in.

Previous work has also applied colored graph models to target tracking in sensor networks \cite{CrespiACM2008}.
Here the nodes are physical locations, the edges are the allowed movements from location to location, and the colors are the sensor signals.
The colored graph formalism allows reasoning about how well a target can be tracked given noisy/ambiguous sensor reports.


\subsection{Observability Classes}
\label{sec:obsClass}
Suppose that an agent is traversing a given node-colored graph, yielding a sequence of observed colors $Y_{1:t}=Y_1,\dots,Y_t$ from which we seek to infer something about the underlying state or node sequence
$X_{1:t}=X_1,\dots,X_t$ that generated
those observed colors.  A ``hypothesis'' in the context of such an observability problem is any
sequence $X_{1:t}=X_1,\dots,X_t$ of nodes that can be visited when traversing a directed path that emits the observed colors $Y_{1:t}=Y_1,\dots,Y_t$.  

Applications often distinguish between real-time tracking where we try to find $X_t$ given $Y_{1:t}$ and \ap reconstruction where we try to infer some part (even all) of the sequence of nodes $X_{1:t}$ given $Y_{1:t}$.  
In addition, different versions of the problem can make different assumptions about whether the start state $X_1$ is known or not.
In general, there is no guarantee that any of the nodes can be unambiguously identified, even \ap.

Previous work has identified a number of classes of colored graphs for which guarantees of varying strength can be made, however:
\begin{itemize}
	\item In a \emph{trackable} graph the number of hypotheses consistent with an observation sequence grows polynomially in the length of the observation sequence \cite{CrespiACM2008}. In a graph which is not trackable, the number of hypotheses grows exponentially.  It is known that the number of hypotheses can grow either polynomially or exponentially, with no intermediate growth rates possible \cite{CrespiACM2008}.
	\item In a \emph{unifilar} graph the current node $X_{t}$ is unambiguously determined given the previous node $X_{t-1}$ and the current color $Y_{t}$ \cite{ShengIEEE2005}. Furthermore, each node emits exactly one color and each color can be emitted by at most one of the starting nodes. (Thus, the start node is determined unambiguously
	by the initial observed color.) The implication of these constraints is that there is a one-to-one correspondence between color sequences and node sequences.
	\item In a \emph{partly {a posteriori} observable} graph it is possible, given a sufficiently long observation sequence, to unambiguously determine the state at at least one point in the past \cite{JungersDAM2011}.
	\item In a \emph{partly observable} graph there is an upper bound, $K$, on how much time there is between opportunities for $X_{t}$ to be unambiguously determined given the observation sequence $Y_{t_0:t}$, where $t_0 \leq t < t_0 + K$ \cite{JungersDAM2011}.
	\item In an \emph{observable} graph, the node $X_t$ can be unambiguously determined given the observation sequence $Y_{1:t}$, provided $t>T$ (where $T$ is a deterministic burn-in period determined by the structure of the graph) \cite{JungersDAM2011}.
\end{itemize}
While the previous work characterized these classes using a variety of approaches and definitions \cite{CrespiACM2008,JungersDAM2011,ShengIEEE2005}, Section~\ref{sec:taxon} shows that they can all be expressed in a common framework.

\subsection{Currency of Estimates}
Previous work on indexing systems (e.g., search engines) has characterized the quality of estimates in terms of $(\alpha,\beta)$-currency \cite{BrewingtonCN2000,BrewingtonC2000,BrewingtonThesis2000}.
Specifically, a quantity which was previously observed at time $t_0$ is said to be $\beta$-current at time $t$ if it has not changed between time $t_0$ and time $t-\beta$.
The quantity is said to be $(\alpha,\beta)$-current if it is $\beta$-current with probability $\alpha$.

Appropriately interpreted for the new domain, $(\alpha,\beta)$-currency provides a useful metric for characterizing the properties of the various observability classes.
Specifically, when tracking an agent traversing a colored graph, our estimate is said to be $(\alpha,\beta)$-current at time $t$ if we can correctly identify the node $X_{t-\beta}$ with probability $\alpha$.
Note that this definition does not make reference to the ``time of last observation,'' $t_0$, because there are no times at which the node is observed directly.
Instead, what matters is how long we have been observing the color sequence: the record length, $\gamma$.
Therefore, we say that an estimate is $(\alpha,\beta,\gamma)$-current if we can correctly identify the node $X_{t-\beta}$ with probability $\alpha$ using the observed colors $Y_{t-\gamma+1},\dots,Y_t$.
This is illustrated in Figure~\ref{fig:abg}.
\begin{figure}
	\centering
	\includegraphics{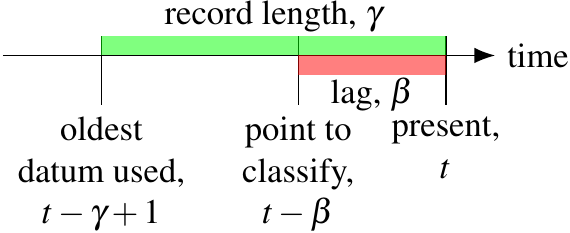}
	\caption{Illustration of $(\alpha,\beta,\gamma)$-currency. The record length $\gamma$ may be set either by the time at which observations started (in which case it grows at each time step) or by the finite number of observations stored in the tracking system's memory (in which case it is constant). The oldest datum used is at $t-\gamma+1$ because we are considering the discrete-time case and the number of observations is $\gamma$.}
	\label{fig:abg}
\end{figure}
As an example, an observable graph is $(1, 0, T+1)$-current.
Because of the coloring constraint on the set of starting nodes, a unifilar graph is $(1, 0, 1)$-current.
Other classes are more complicated; attainable values of $\alpha$, $\beta$, and $\gamma$ depend on the specific structure of the graph.

\section{Colored Graph Pathologies}
\label{sec:pathos}

\subsection{Description of the Pathologies}
Jungers and Blondel present polynomial-time algorithms to check whether a graph is observable or partly \ap observable by checking for the following properties \cite{JungersDAM2011}:
\newcounter{listHandoff}
\begin{enumerate}
	\item Presence/absence of nodes which have out-neighbors of the same color.\label{item:scon}
	\item Presence/absence of separated cycles having the same sequence of colors. Two cycles $\pi_1,\pi_2:\mathbb{Z}\to V$ indexed by $i$ and permitting the same sequence of colors (i.e., $L(\pi_1(i))\cap L(\pi_2(i))\neq\emptyset$ $\forall i$) are said to be \emph{separated} if $\pi_1(i)\neq\pi_2(i)$ for all steps $i$. (It is possible, however, to have $\pi_1(i)=\pi_2(j)$ for $i\neq j$; see Figure~\ref{fig:sepCyc}.)
	\setcounter{listHandoff}{\value{enumi}}
\end{enumerate}
In Section~\ref{sec:trackDef}, we show that the polynomial-time algorithm given by Crespi et al.\ to check whether a graph is trackable is equivalent to checking for the following additional property \cite{CrespiACM2008}:
\begin{enumerate}
	\setcounter{enumi}{\value{listHandoff}}
	\item Presence/absence of intersecting cycles having the same sequence of colors. Two cycles $\pi_1(i)$, $\pi_2(i)$ permitting the same sequence of colors are said to be \emph{intersecting} if there is at least one $i$ such that $\pi_1(i)=\pi_2(i)$. (We exclude the trivial case of identical cycles where $\pi_1(i)=\pi_2(i)$ $\forall i$.)
\end{enumerate}
\begin{figure}
	\centering
	\includegraphics[scale=0.15]{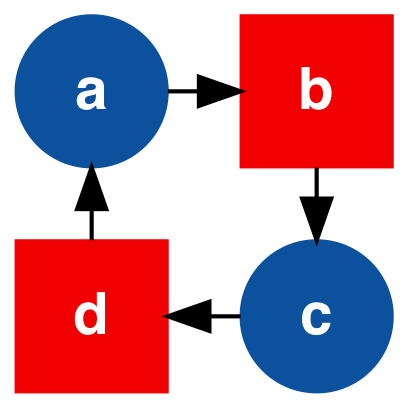}
	\caption{Graph with \emph{separated} cycles which share nodes. The cycles $\pi_1=(a, b, c, d, a,\dots)$ and $\pi_2=(c, d, a, b, c,\dots)$ have the same sequence of colors and involve the same nodes, but there is no $i$ such that $\pi_1(i)=\pi_2(i)$.}
	\label{fig:sepCyc}
\end{figure}
Note that intersecting cycles with the same coloring are a subset of same-colored out-neighbors: because a pair such cycles must intersect, there must be some $i$ such that $\pi_1(i)=\pi_2(i)$ and $\pi_1(i+1)\neq\pi_2(i+1)$.
(Recall that we exclude identical cycles.)
The fact that $\pi_1$ and $\pi_2$ are taken to permit the same sequence of colors means that $\pi_1(i+1)$ and $\pi_2(i+1)$ are same-colored out-neighbors of the branching point $\pi_1(i)=\pi_2(i)$.
But, as shown by the examples in Figure~\ref{fig:taxonExamples}, not all same-colored out-neighbors are part of intersecting cycles with the same coloring.

This enumeration suggests that the three pathologies listed in Table~\ref{tab:pathos} are useful for characterizing the various observability classes.
\begin{table*}
	\centering
	\caption{Colored Graph Pathologies, Examples, Effects, and Mitigations}
	\makebox[\textwidth][c]{%
	\begin{tabular}{p{1in}cp{1.5in}c}
		\toprule
		Name & Example & Effect & Mitigation\\
		\midrule
		Same-colored out-neighbors & \raisebox{-0.8\totalheight}{\includegraphics[scale=0.1]{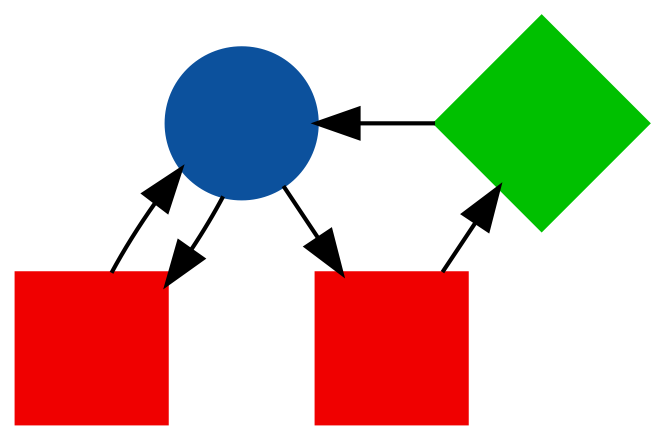}} & Lose track, may be able to reconstruct \emph{a posteriori}, increases $\beta$ for given $\alpha$ & \raisebox{-0.8\totalheight}{\includegraphics[scale=0.1]{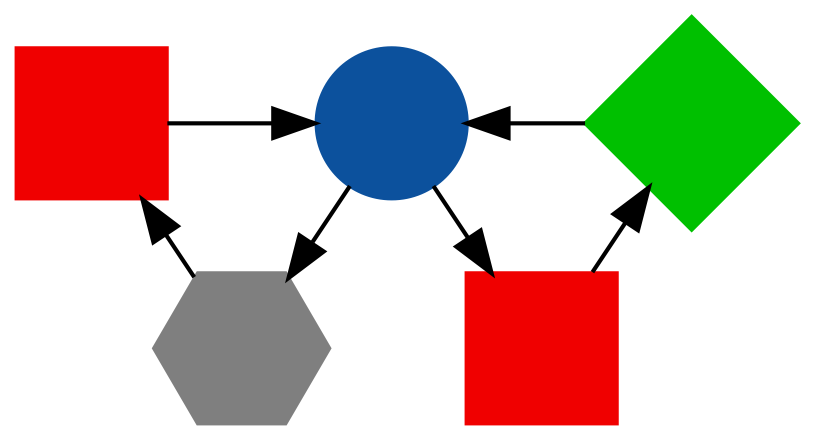}}\\
		\midrule
		Intersecting cycles with same coloring & \raisebox{-0.8\totalheight}{\includegraphics[scale=0.1]{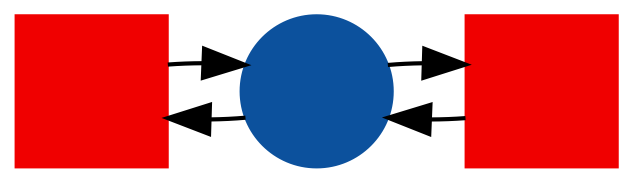}} & Lose track, may never be able to reconstruct, decreases maximum attainable $\alpha$ & \raisebox{-0.8\totalheight}{\includegraphics[scale=0.1]{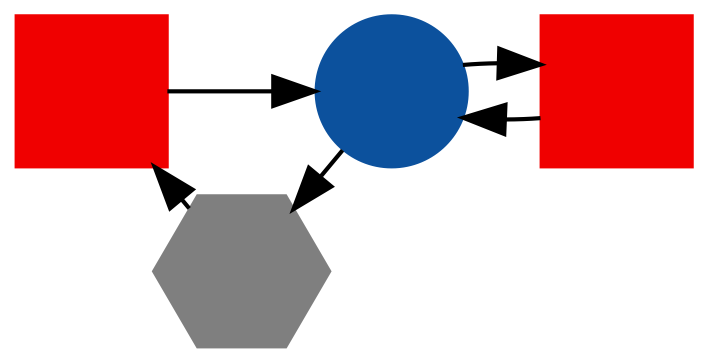}}\\
		\midrule
		Separated cycles with same coloring & \raisebox{-0.8\totalheight}{\includegraphics[scale=0.1]{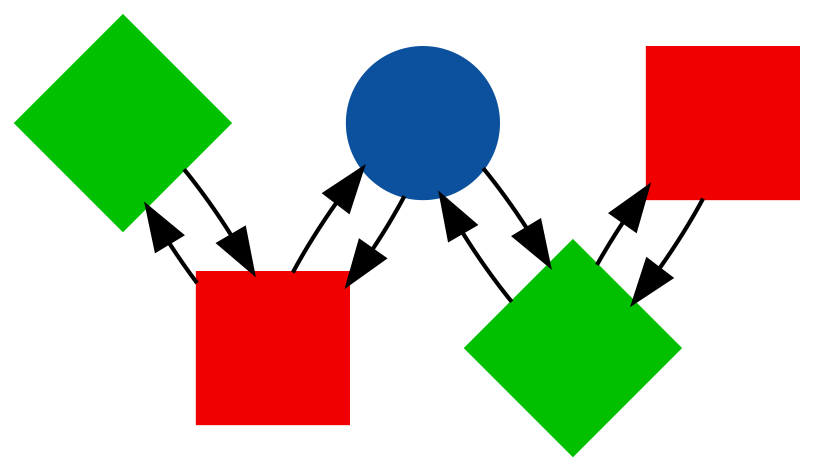}} & Increases ``burn-in'' time before nodes can be reconstructed unambiguously, increases required $\gamma$ for a given $\alpha$ & \raisebox{-0.8\totalheight}{\includegraphics[scale=0.1]{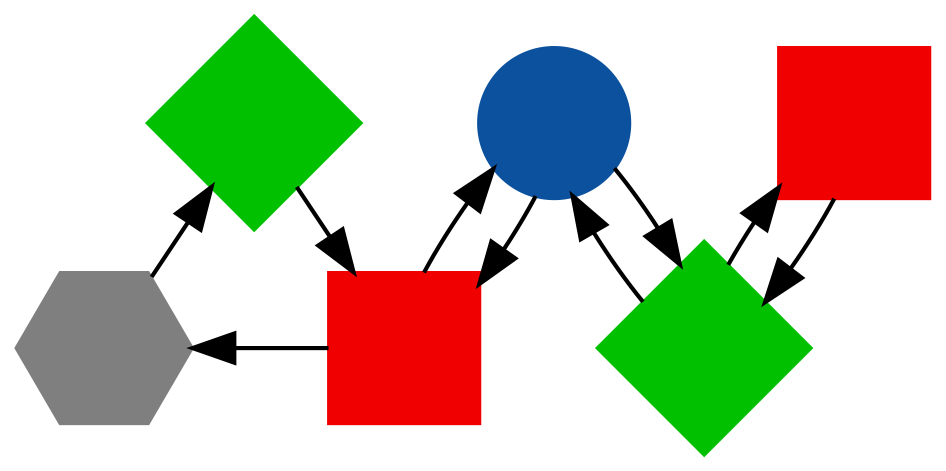}}\\
		\bottomrule
	\end{tabular}
	}
	\label{tab:pathos}
\end{table*}
In fact, in Section~\ref{sec:taxon} we show that these three pathologies are sufficient to describe observable, partly \ap observable, trackable, and a looser class of unifilar graphs (i.e., without the constraint on the set of starting nodes).
Partly observable graphs do not fit quite as cleanly into this framework as the others, but some cases can be characterized by the absence of a specific type of same-colored out-neighbor described in Section~\ref{sec:PartObs}.

\subsection{Effects of the Pathologies}
In loose terms, the effects of each pathology are:
\begin{itemize}
	\item \emph{Same-colored out-neighbors} cause tracking to be lost once they are encountered. But, it will often be possible to reconstruct which branch was taken \ap. For example, in the top row of Table~\ref{tab:pathos}, once \redsquare is seen, we do not know which node the agent is at. But, the next observation will either be \bluecircle (in which case we know the previous step took the left-hand branch) or \greendiamond (in which case we know the previous step took the right-hand branch). Therefore, the net effect on $(\alpha,\beta,\gamma)$-currency is to increase the lag $\beta$ necessary to obtain a given accuracy $\alpha$. Provided that $\gamma>\beta$, the record length $\gamma$ will have no effect on the ability to determine which branch was taken.
	\item \emph{Intersecting cycles with the same coloring} is a special case of same-colored out-neighbors which causes tracking to be lost in a way which can prevent even \ap reconstruction of the visited nodes. For example, in the middle row of Table~\ref{tab:pathos}, the observations will always be an alternating sequence of \redsquare and \bluecircle, but it will never be possible to determine which of the \redsquare nodes was visited. Therefore, the net effect on $(\alpha,\beta,\gamma)$-currency is to decrease the accuracy $\alpha$ which can be obtained for any lag $\beta$ or record length $\gamma$.
	\item \emph{Separated cycles with the same coloring} increase the ``burn-in'' time for which colors must be observed before a node (past or present) can be identified unambiguously. For example, consider a sequence of observations from the graph in the bottom row of Table~\ref{tab:pathos} consisting of alternating \greendiamond and \redsquare. Until \bluecircle is observed, it is not possible to know which side of the graph the agent is on. But, as soon as one of the sequences (\redsquare, \bluecircle) or (\greendiamond, \bluecircle) is observed, we know not only where the agent is, but where it was at all previous times.
\end{itemize}
When there is no risk of confusion, the words ``with the same coloring'' will be omitted when referring to intersecting and separated cycles.

\subsection{Mitigating the Pathologies}
Enumerating the pathologies is useful not just to understand their implications, but also to create ways of mitigating their deleterious effects on tracking performance.
A simple way to modify a colored graph's observability class without significantly changing the functionality of the underlying system is to add uniquely-colored but otherwise non-functional ``indicator nodes'' at strategic locations.
These are indicated by the grey hexagons (\greyhex) in the last column of Table~\ref{tab:pathos}.
This approach was demonstrated experimentally in \cite{ChilenskiSPIE2018}, and further simulated examples are given in Section~\ref{sec:examples}.

The basic idea is to insert an indicator node either just before a same-colored out-neighbor or in a cycle with the same coloring in order to remove the pathology.
Because the addition of an indicator node causes a delay in the traversal of the graph, it is desirable to put the indicator nodes in parts of the graph which are less-frequently visited.
In an HMM, one can determine the long-run frequency of each transition in order to decide where to place indicator nodes.

In principle it is also possible to change the observability class through the deletion of nodes and/or edges, instead of the insertion of nodes into existing edges discussed above.
This could correspond to restructuring the system to avoid certain ambiguous behaviors.
But, because this will clearly affect the functionality of the system more than the addition of indicator nodes, we do not consider these cases any further here.

\section{Pathology-Based Taxonomy of Colored Graph Observability Classes}
\label{sec:taxon}

A taxonomy of colored graph classes based on the presence/absence of the graph pathologies is given in Figure~\ref{fig:taxon}.
Simple example graphs from each region of the Venn diagram are given in Figure~\ref{fig:taxonExamples}.
The full reasoning for each class/region is given in the following subsections.

\begin{figure*}
	\centering
	\includegraphics{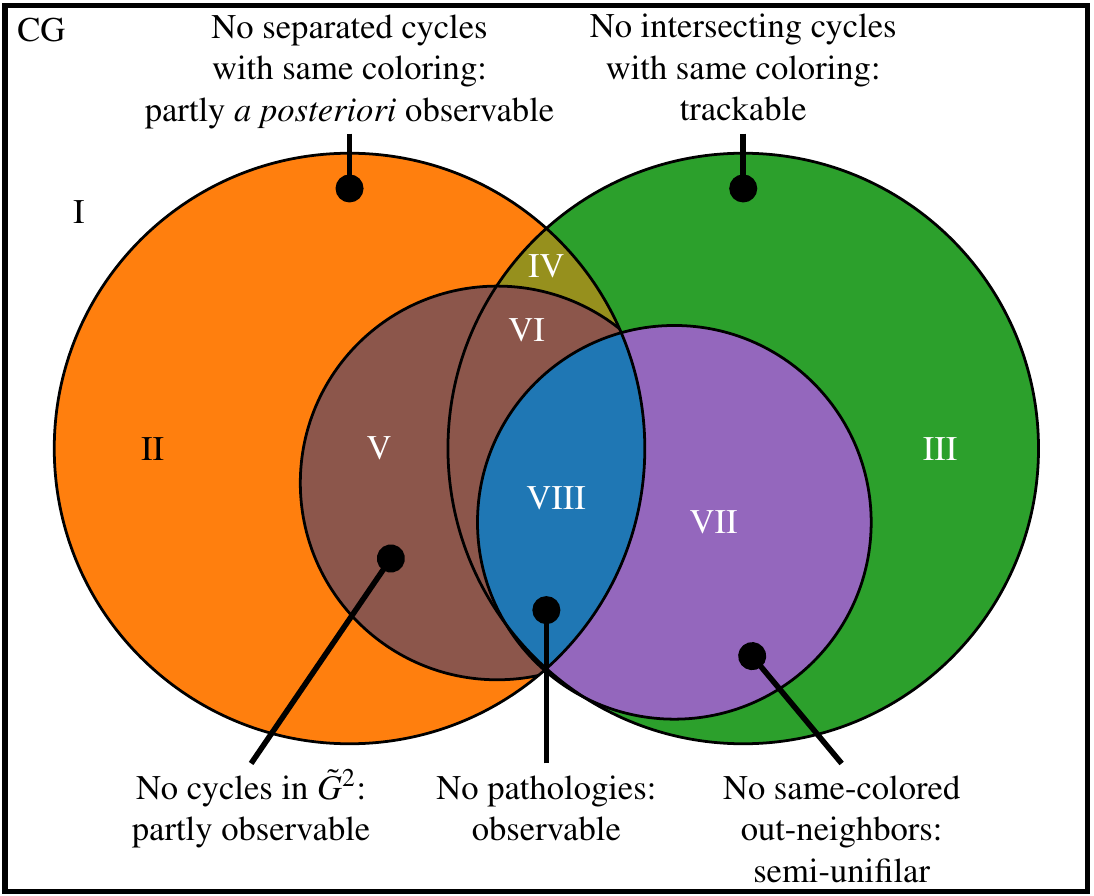}
	\caption{Venn diagram depicting the taxonomy of colored graphs. Roman numerals in each portion of the diagram are used to refer to the different portions throughout the text. The ``CG'' in the upper left (region I) stands for ``colored graph'' -- the universe of all possible colored graphs.}
	\label{fig:taxon}
\end{figure*}

\begin{figure}
	\centering
	\subfloat[Region I: all pathologies]{\includegraphics[scale=0.125]{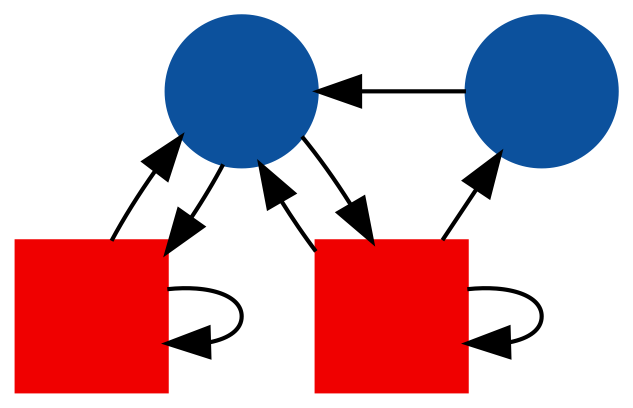}}
	\quad\quad%
	\subfloat[Region II: no separated cycles: partly \ap observable]{\includegraphics[scale=0.125]{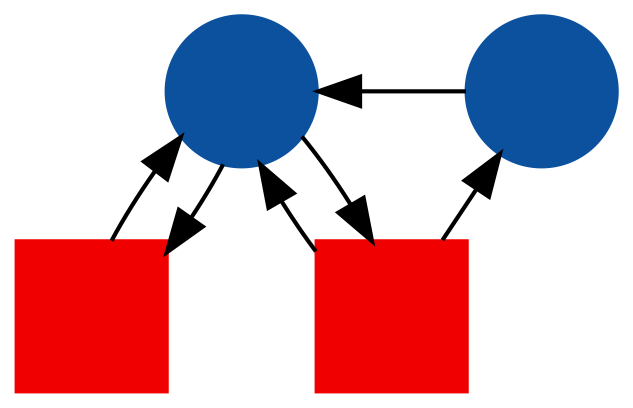}}
	
	\subfloat[Region III: no intersecting cycles: trackable]{\includegraphics[scale=0.125]{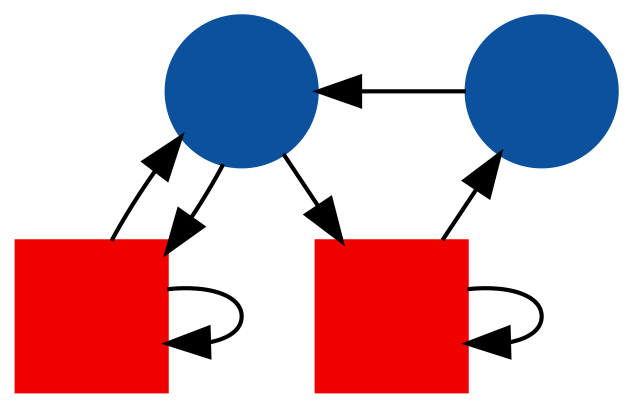}}
	\quad\quad%
	\subfloat[Region IV: no separated or intersecting cycles: partly \ap observable and trackable]{\includegraphics[scale=0.125]{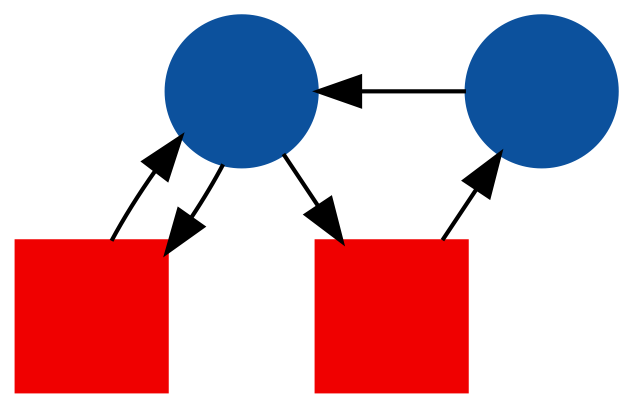}\label{sf:r4}}
	
	\subfloat[Region V: acyclic $\tilde{G}^2$: partly observable]{\includegraphics[scale=0.125]{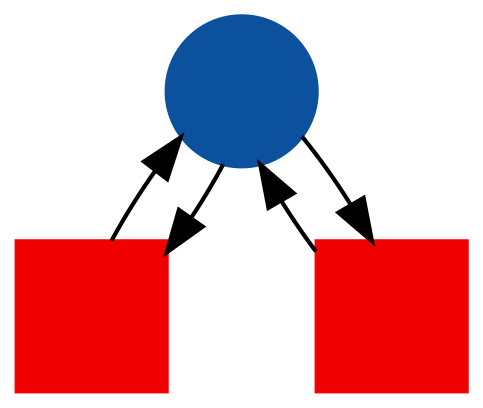}\label{sf:r5a}}
	\quad\quad%
	\subfloat[Region VI: acyclic $\tilde{G}^2$ and no intersecting cycles: partly observable and trackable]{\includegraphics[scale=0.125]{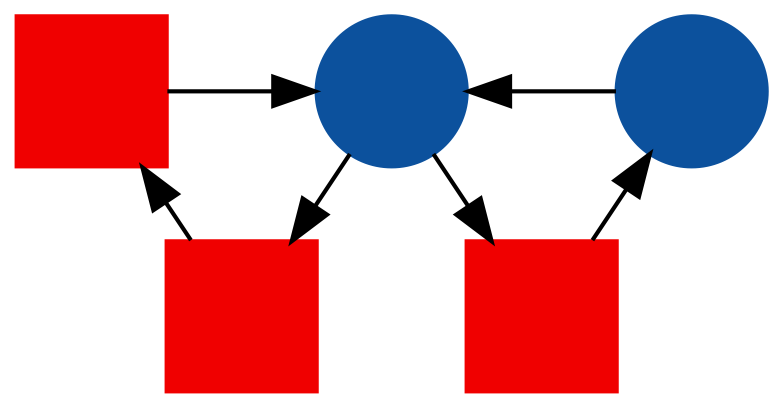}}
	
	\subfloat[Region VII: no same-colored out-neighbors: semi-unifilar]{\includegraphics[scale=0.125]{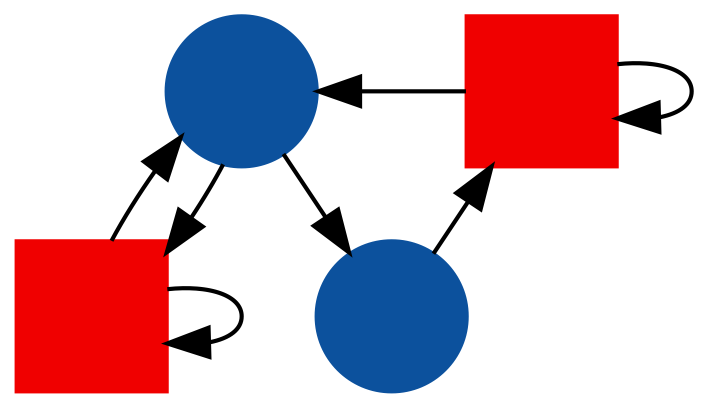}\label{sf:r6}}
	\quad\quad%
	\subfloat[Region VIII: no pathologies: observable]{\includegraphics[scale=0.125]{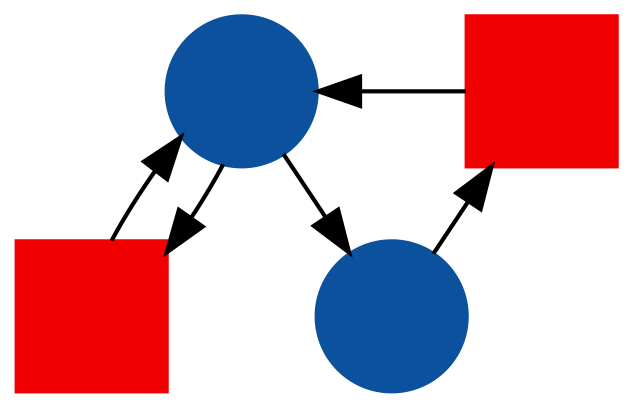}}
	
	\caption{Example graphs for each region of Figure~\ref{fig:taxon}.}
	\label{fig:taxonExamples}
\end{figure}

\subsection{Region I: General Colored Graphs}
The outer part of the Venn diagram is the universe of all possible colored graphs, which may have all of the pathologies represented, and for which no performance guarantees can be made.

\subsection{Region II: Partly \emph{a Posteriori} Observable}
A graph is partly \ap observable if there are no separated cycles with the same color sequence \cite{JungersDAM2011}.

To check if a colored graph $G=(V, E, L, \Phi)$ possesses this property, construct the auxiliary graph $G^2$ whose nodes are of the form $(v_1, v_2)$, where $v_1,v_2\in V$ and $v_1\neq v_2$.
The auxiliary graph contains an edge $((v_1, v_2), (v_1', v_2'))$ if $(v_1, v_1')\in E$, $(v_2, v_2')\in E$, and $v_1'$ and $v_2'$ have the same color (i.e., $L(v_1')\cap L(v_2')\neq\emptyset$).
If $G^2$ is acyclic, then $G$ contains no separated cycles with the same color sequence \cite{JungersDAM2011}.

\subsection{Region III: Trackable}
\label{sec:trackDef}
Crespi et al.\ characterize trackability by considering the node sequences consistent with all possible color sequences \cite{CrespiACM2008}.
A graph is trackable if and only if, for each possible color sequence, there is at most one path from each node $v$ at time $t_1$ back to itself at time $t_2$ which is consistent with the color sequence (this is the ``unique path property'').
But, if two different paths begin and end at the same node and have the same color sequence, then they form a pair of intersecting cycles with the same coloring.
Therefore, the absence of such cycles is a necessary and sufficient condition for a graph to be trackable.

\subsection{Region IV: Partly \emph{a Posteriori} Observable and Trackable}
As shown by the example in Figure~\ref{sf:r4}, it is possible for a graph to lack both separated and intersecting cycles with the same coloring but to not belong to any of the more restrictive classes.
Namely, the absence of both separated and intersecting cycles with the same color sequence is not a sufficient condition for a graph to be partly observable.
(Nor is it a necessary condition, see Figure~\ref{sf:r5a}.)

\subsection{Regions V and VI: Partly Observable}
\label{sec:PartObs}
Partial observability is characterized by another auxiliary graph, $\tilde{G}^2$ \cite{JungersDAM2011}.
To construct $\tilde{G}^2$, add the edge $((v_1, v_2), (v_1', v_2'))$ to $G^2$ if the following three conditions are met:
\begin{enumerate}
	\item $G$ has edge $(v_1, v_1')$ or $(v_2, v_1')$
	\item $G$ has edge $(v_1, v_2')$ or $(v_2, v_2')$
	\item $v_1'$ and $v_2'$ have the same color (i.e., $L(v_1')\cap L(v_2')\neq \emptyset$)
\end{enumerate}
A graph is partly observable if $\tilde{G}^2$ is acyclic.
Because $\tilde{G}^2$ is a supergraph of $G^2$, it has at least as many cycles as $G^2$ and hence partly observable is a subset of partly \ap observable.

Graphs which are partly \ap observable but not partly observable (i.e., $G^2$ is acyclic but $\tilde{G}^2$ is not) appear to be characterized by a specific class of same-colored out-neighbors similar to the example in Figure~\ref{sf:r4}: there is a cycle connected to a path which permits the same sequence of colors as the cycle, but the path ends at a different node and hence does not form an intersecting cycle.
The net effect of this configuration is to permit color sequences of arbitrary length with ambiguous endpoints, thereby violating the conditions for partial observability.
Specifically, in Figure~\ref{sf:r4}, we cannot know in real-time when a sequence of the form (\bluecircle, \redsquare, \bluecircle, \redsquare, \dots) has transitioned from the left-hand branch to the right-hand branch.
But, once we see the sequence (\bluecircle, \redsquare) twice in a row, we know that the \emph{previous} two nodes were in the left-hand branch, but remain uncertain about which branch the agent is currently on.
Therefore, the net effect of this pathology is simply to increase the lag $\beta$ for correct identification of nodes.
Because the effect on $(\alpha,\beta,\gamma)$-currency is identical to the more general class of same-colored out-neighbors which do not form intersecting cycles, we have chosen to not include this as a specific pathology in Table~\ref{tab:pathos}, but simply include partly observable as a subset in our taxonomy.

As noted in \cite{JungersDAM2011}, partly observable graphs may or may not be trackable.
We have designated these cases regions V and VI, respectively.

\subsection{Region VII: Semi-Unifilar}
A graph $G=(V, E, L, \Phi)$ is unifilar if the following three conditions are met \cite{ShengIEEE2005}:
\begin{enumerate}
	\item Each node $v\in V$ emits exactly one color: $|L(v)|=1$.\label{item:unif1}
	\item For each node $v\in V$ and each color $c\in\Phi$, there is at most one out-neighbor of $v$ which can emit $c$.\label{item:unif2}
	\item For each color $c\in\Phi$, the agent can start at at most one node which emits $c$.\label{item:unif3}
\end{enumerate}
Condition~\ref{item:unif1} is not necessary for the favorable tracking properties of unifilar graphs described in Section~\ref{sec:obsClass}.
(It is used in \cite{ShengIEEE2005} to establish bounds on the rate of growth of the set of possible color sequences from weak models.)
Condition~\ref{item:unif2} simply states that there are no same-colored out-neighbors.
Our taxonomy does not consider the set of nodes which the agent is permitted to start at, so we do not consider condition~\ref{item:unif3}.
Therefore, we call a graph \emph{semi-unifilar} if it satisfies condition~\ref{item:unif2} (no same-colored out-neighbors).

Semi-unifilar graphs have the property that, once a node is identified unambiguously, all nodes from then on will also be identified unambiguously.
But, semi-unifilar graphs do not have any guarantees that you will be able to perform this initial localization.
For example, the graph in Figure~\ref{fig:sepCyc} is semi-unifilar but, because of its symmetry, it will never be possible to unambiguously identify any nodes without additional information beyond the observed color sequence.

Because intersecting cycles with the same coloring are a specific case of same-colored out-neighbors, semi-unifilar is clearly a subset of trackable.
Contradicting the assertion that unifilar graphs are a particular case of observable graphs in Section~1 of \cite{JungersDAM2011}, we note that unifilar (and semi-unifilar) graphs can have separated cycles with the same coloring, and are therefore in fact a superset of observable graphs.
An example of a semi-unifilar graph which is not observable is given in Figure~\ref{sf:r6}.

\subsection{Region VIII: Observable}
A graph is observable if it lacks both separated cycles with the same coloring and same-colored out-neighbors \cite{JungersDAM2011}.
In other words, observable graphs are pathology-free.

\section{Optimal Indicator Node Placement Is NP-Complete}
\label{sec:INSP}
In this section, we consider the problem of making an arbitrary graph partly \ap observable by adding some number of indicator nodes.
We show that this problem is in NP-Complete, but leave open the questions of obtaining other observability-type properties through the
insertion of indicator nodes.
Recall that adding an indicator node means replacing an edge $(v_i,v_j)$
with a new node $v_{ij}$ and new edges $(v_i,v_{ij})$ and $(v_{ij},v_j)$.  Table~\ref{tab:pathos} illustrates the concept with three examples.
To be precise, consider the following formulation of the problem.  

\begin{quote}
{\em Indicator Node Selection Problem}: Given a node-colored directed graph, $G=(V,E,L,\Phi)$, and a subset of edges, $F \subseteq E$, can indicator nodes be added to some
edges in $F$ so that the resulting graph is partly \ap observable?
\end{quote}

\begin{theorem}
	The Indicator Node Selection Problem is in NP-Complete. \label{INSP-theorem}
\end{theorem}

The proof is deferred to the Appendix (part of the supplemental material) because the reduction is detailed and would detract from the
present narrative.  But we note here that the very question of whether inserting
indicator nodes into a given subset of edges can make the resulting graph partly observable is in NP-Complete
without requiring that the number of added indicator nodes be minimal.  Moreover, the proof of Theorem~\ref{INSP-theorem} will show that the result is true irrespective of what color the indicator nodes are -- 
they can be one of the existing colors in $L$ or an entirely new color.

However, the proof does not address the problem of adding indicator nodes with arbitrary numbers
of new colors nor the problem of repeatedly adding indicator nodes to the same original edge
because those cases are trivial.
To see this, note that every edge in $G$ could have a single uniquely colored indicator node color inserted or
a unique number of same colored indicator nodes inserted to make each edge effectively uniquely colored.

\section{Numerical Experiments: Changing Graph Class With Indicator Nodes}
\label{sec:examples}
In order to illustrate the effects of the pathologies and mitigations, we have conducted a series of numerical experiments using the graph shown in Figure~\ref{sf:base}.
This graph has all three pathologies present, and hence is expected to have poor tracking performance.

\begin{figure*}
	\centering
	\makebox[\textwidth][c]{
	\subfloat[Base case: all pathologies]{\includegraphics[scale=0.15]{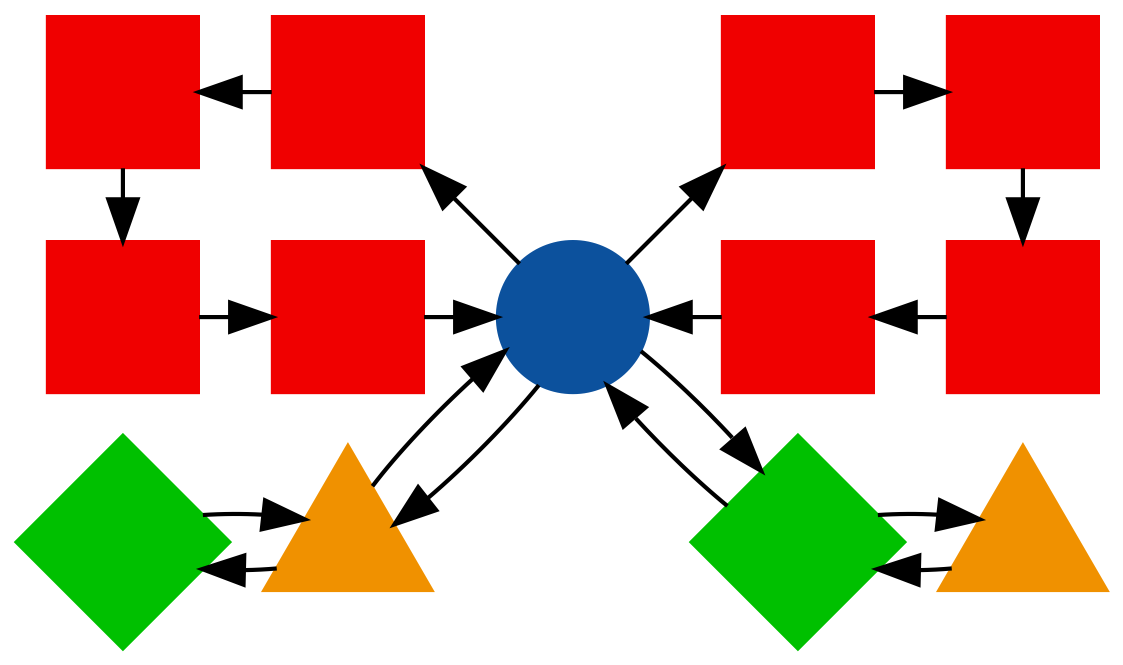}\label{sf:base}}
	\quad\quad
	\subfloat[Mitigate intersecting cycles: trackable]{\includegraphics[scale=0.15]{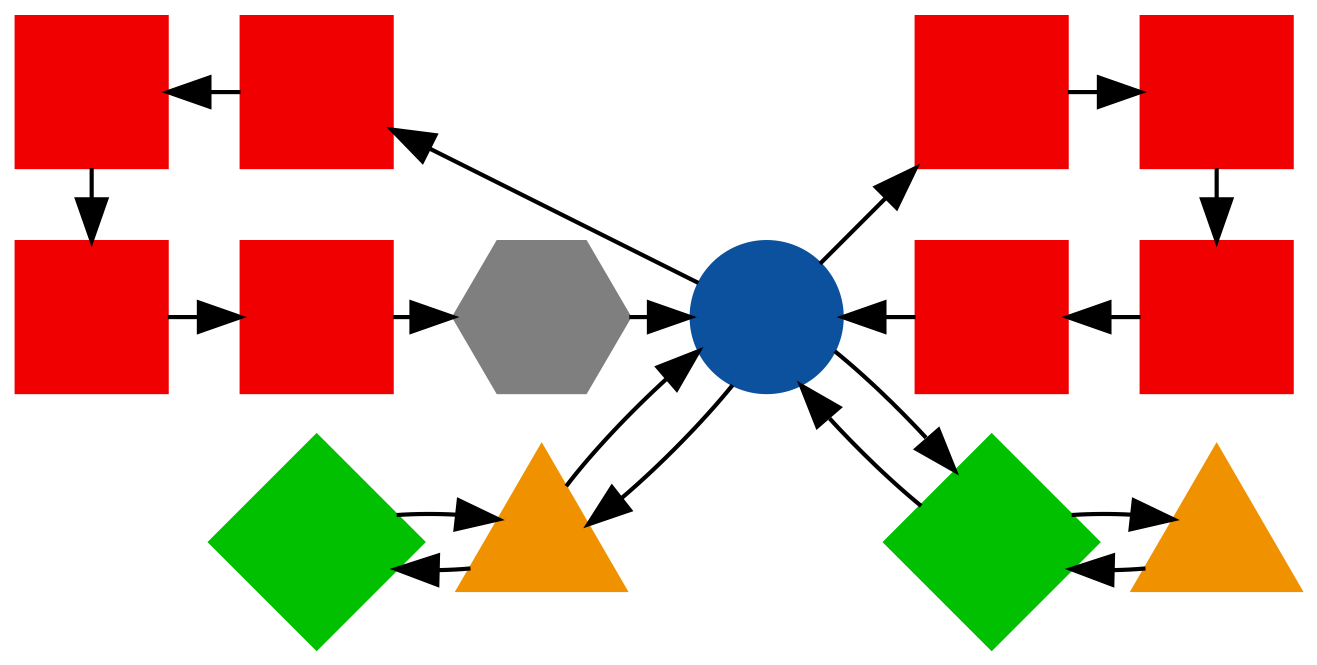}\label{sf:trackMit}}
	}
	
	\makebox[\textwidth][c]{
	\subfloat[Mitigate same-colored out-neighbors: semi-unifilar]{\includegraphics[scale=0.15]{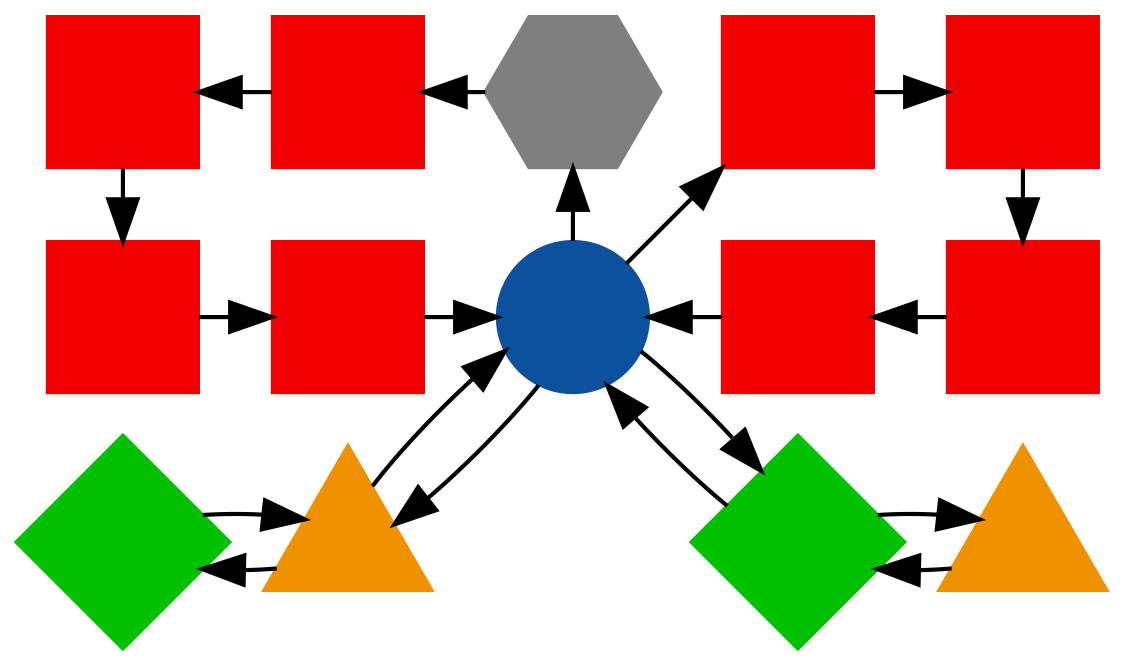}\label{sf:unifMit}}
	\quad\quad
	\subfloat[Mitigate separated cycles: observable]{\includegraphics[scale=0.15]{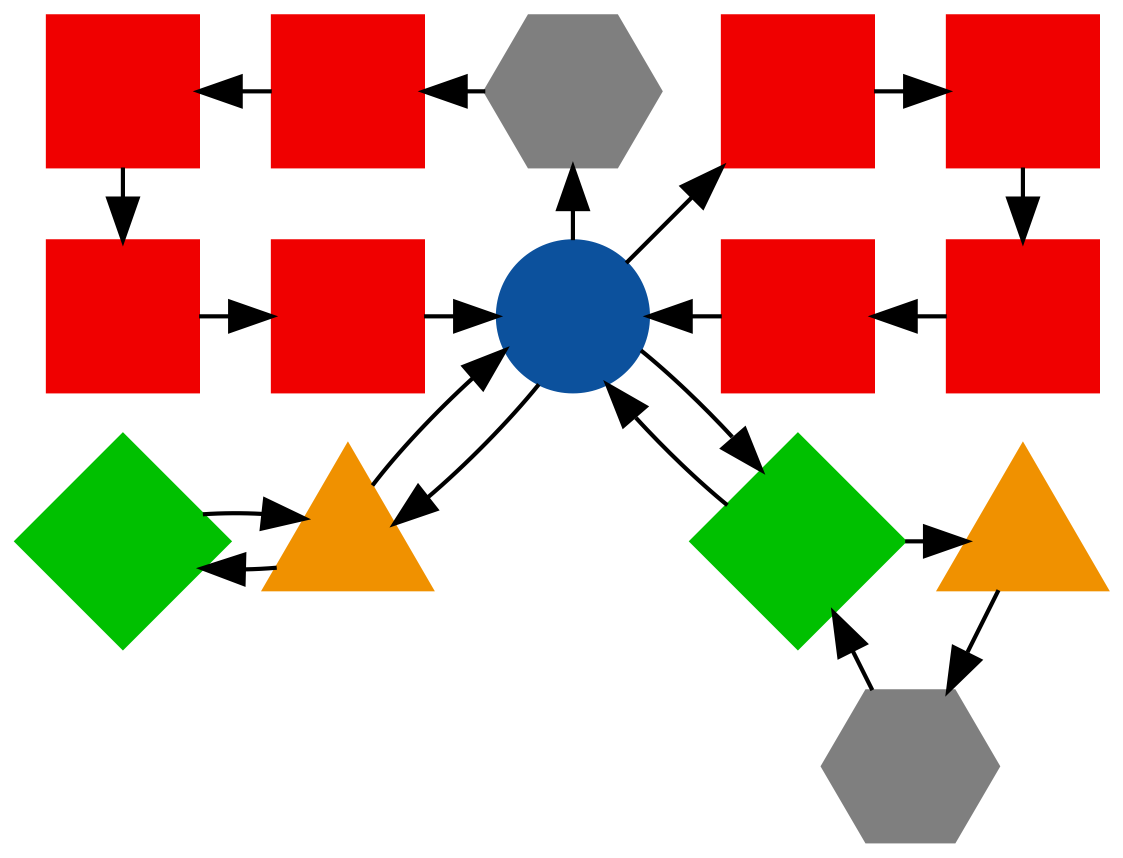}\label{sf:obsMit}}
	}
	\caption{Successive mitigations of colored graph pathologies.}
	\label{fig:butterfly}
\end{figure*}

\subsection{Improving Tracking Performance: Trackable and Semi-Unifilar}
The base case shown in Figure~\ref{sf:base} has a pair of intersecting cycles of the form (\bluecircle, \redsquare, \redsquare, \redsquare, \redsquare).
These will prevent reconstruction of which nodes were visited: every time the sequence (\bluecircle, \redsquare) is observed, the size of the hypothesis set doubles, consistent with the exponential growth expected for an untrackable graph.

To illustrate this, we simulated \num{10000} draws of 50 steps each from an HMM defined by the graph shown in Figure~\ref{sf:base}.
The probabilities of transitions out of each node were set to be equal, and the nodes were taken to be single-colored.
To capture the steady-state behavior, we set the initial state distribution of the HMM to be equal to the equilibrium distribution.
We then used the Viterbi algorithm to reconstruct the node sequence from the color sequence with various lags $\beta$ and record lengths $\gamma$.

The accuracy $\alpha=P(\hat{X}_{t-\beta}=X_{t-\beta})$ (where $\hat{X}_t$ and $X_t$ are the predicted and true nodes at time $t$, respectively) is shown in Figure~\ref{fig:steadyStateContour}a.
The median accuracy for the base graph is 78\%, and drops to 50\% for very short record lengths.
The steady-state (i.e., high-$\gamma$) behavior is shown in Figure~\ref{fig:steadyStateAlpha}.
The tracking accuracy is equally poor at all lags: longer record lengths and/or lags do not help reduce the effects of intersecting cycles with the same coloring.

\begin{figure*}
	\centering
	\makebox[\textwidth][c]{\includegraphics{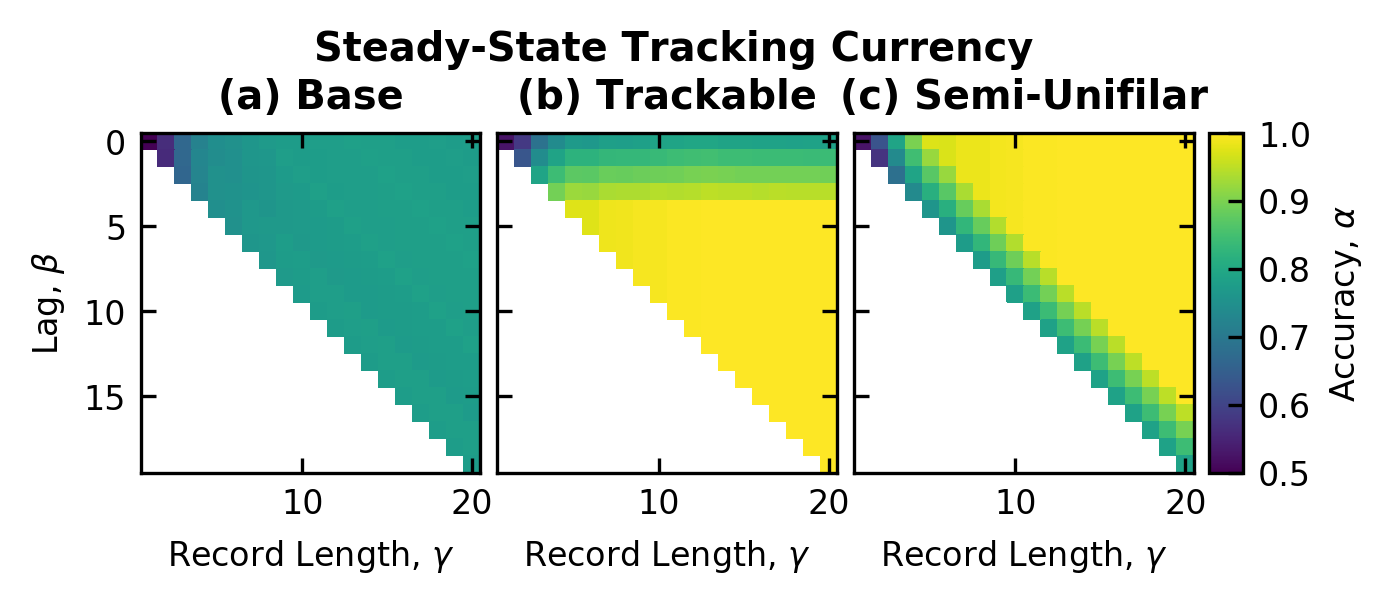}}
	\caption{Accuracy $\alpha=P(\hat{X}_{t-\beta}=X_{t-\beta})$ as a function of lag $\beta$ and record length $\gamma$ for the (a) basic, (b) trackable, and (c) semi-unifilar graphs. The white regions correspond to $\beta \geq \gamma$, for which the behavior is undefined.}
	\label{fig:steadyStateContour}
\end{figure*}

\begin{figure}
	\centering
	\includegraphics{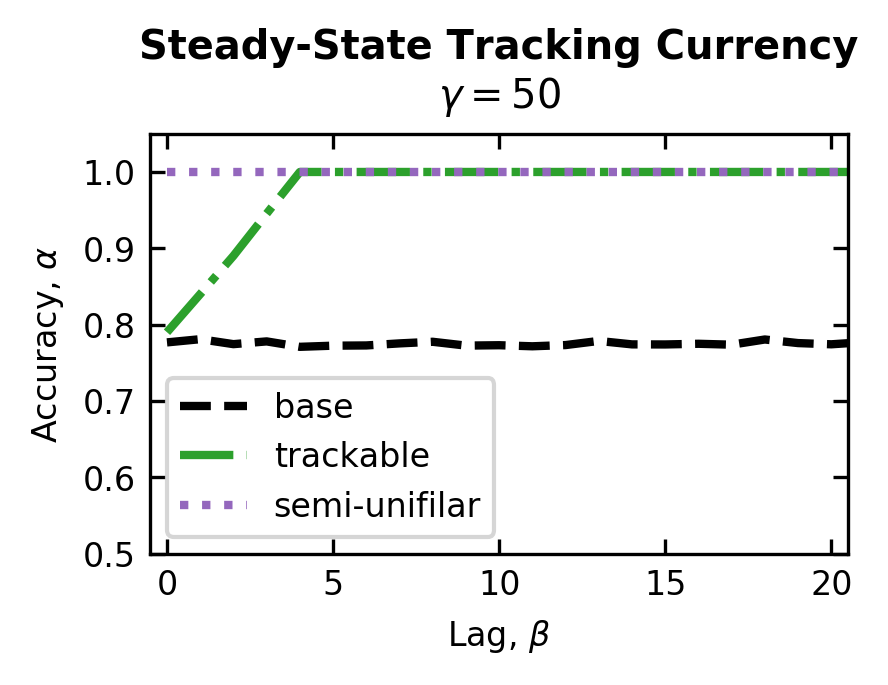}
	\caption{Steady-state (record length $\gamma=50$) tracking accuracy for the basic (black dashed), trackable (green dot-dash) and semi-unifilar (purple dotted) graphs. This is essentially a vertical slice of the data shown in Figure~\ref{fig:steadyStateContour}. The base graph starts around 77\% accuracy and never improves even for very long lags. The trackable graph has 100\% accuracy for lags $\beta\geq4$, but drops to the base 77\% level for shorter lags because of the same-colored out-neighbors of the form (\protect\bluecircle, \protect\redsquare). The semi-unifilar graph has 100\% accuracy for the lags shown, including real-time tracking ($\beta=0$).}
	\label{fig:steadyStateAlpha}
\end{figure}

Now consider the addition of an indicator node to mitigate the intersecting cycles, as shown in Figure~\ref{sf:trackMit}.
The indicator node was placed just \emph{before} the central node to preserve the same-colored out-neighbors: the modified graph is trackable, but not semi-unifilar.
Tracking accuracy for the modified graph (Figure~\ref{fig:steadyStateContour}b) is 100\% for $\beta\geq4$, but the same-colored out-neighbors cause it to drop to around 79\% for $\beta=0$.

Next, consider moving the indicator node so it also mitigates the same-colored out-neighbors, as shown in Figure~\ref{sf:unifMit}.
Tracking accuracy for the modified graph (Figure~\ref{fig:steadyStateContour}c) is 100\% for $\beta<\gamma - 4$, but the accuracy drops for the first four time steps recorded because there is no way to determine which branch an initial sequence of \redsquare came from.

\subsection{Reducing Burn-In Time: Observable}
The graphs considered above contain a pair of separated cycles of the form (\greendiamond, \orangetriangle) which are expected to increase the burn-in time before states can be identified unambiguously.
The stationary distribution for the examples above has most of its mass in the eight \redsquare nodes, which ends up masking this effect.
In order to characterize the effect of separated cycles, we generated another set of \num{10000} realizations, each with 50 time steps.
For this experiment the initial state distribution was set to be uniform over the lower four nodes in Figure~\ref{sf:base} (i.e., the ones which form the separated cycles with the same coloring).

We start with the semi-unifilar graph in Figure~\ref{sf:unifMit} because the effects of the other pathologies have already been shown above.
Tracking accuracy for this case is shown in Figure~\ref{fig:burnInContour}a.
The accuracy starts out poor for all lags $\beta$, but gradually improves as the record length $\gamma$ increases.
The lack of dependence on the lag $\beta$ is expected: the graph is semi-unifilar, so once \bluecircle is observed for the first time unambiguous real-time tracking is guaranteed.
Furthermore, because the observations cannot start at any of the \redsquare nodes, all of the initial states are reconstructed once one of the two sequences (\greendiamond, \bluecircle) or (\orangetriangle, \bluecircle) is observed.

The real-time ($\beta=0$) tracking performance is shown in Figure~\ref{fig:burnInAlpha}.
The tracking accuracy for the semi-unifilar graph asymptotically approaches 100\%.
The time constant of this approach can be controlled by varying the probabilities of transitioning to the \bluecircle node; the same equal probabilities used above were used here.

\begin{figure}
	\centering
	\includegraphics{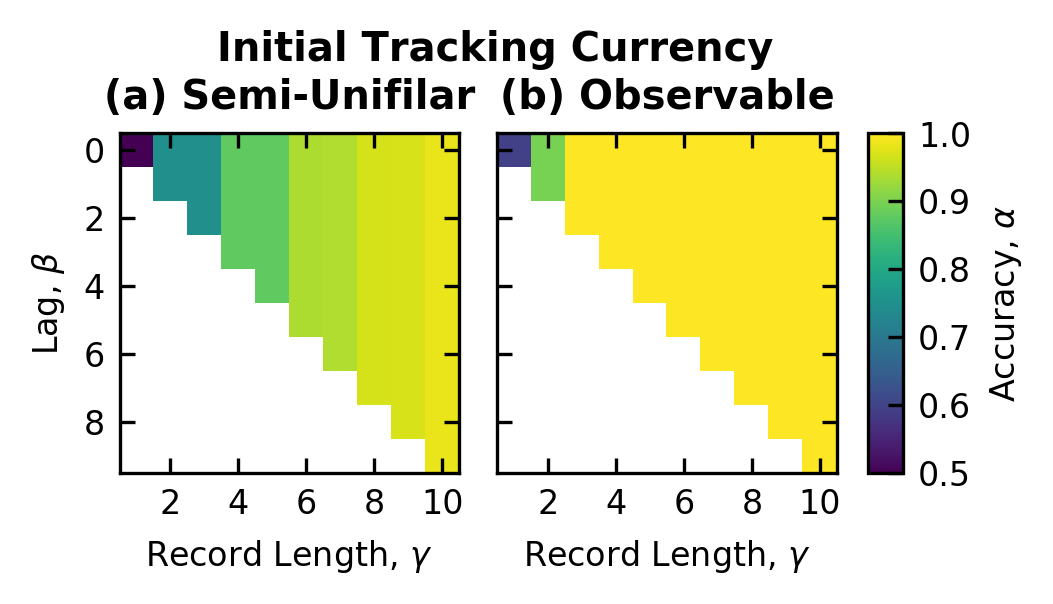}
	\caption{Accuracy $\alpha=P(\hat{X}_{t-\beta}=X_{t-\beta})$ as a function of lag $\beta$ and record length $\gamma$ for the (a) semi-unifilar and (b) observable graphs when observations start while the agent is at either a \protect\greendiamond or \protect\orangetriangle node. The white regions correspond to $\beta \geq \gamma$, for which the behavior is undefined.}
	\label{fig:burnInContour}
\end{figure}

\begin{figure}
	\centering
	\includegraphics{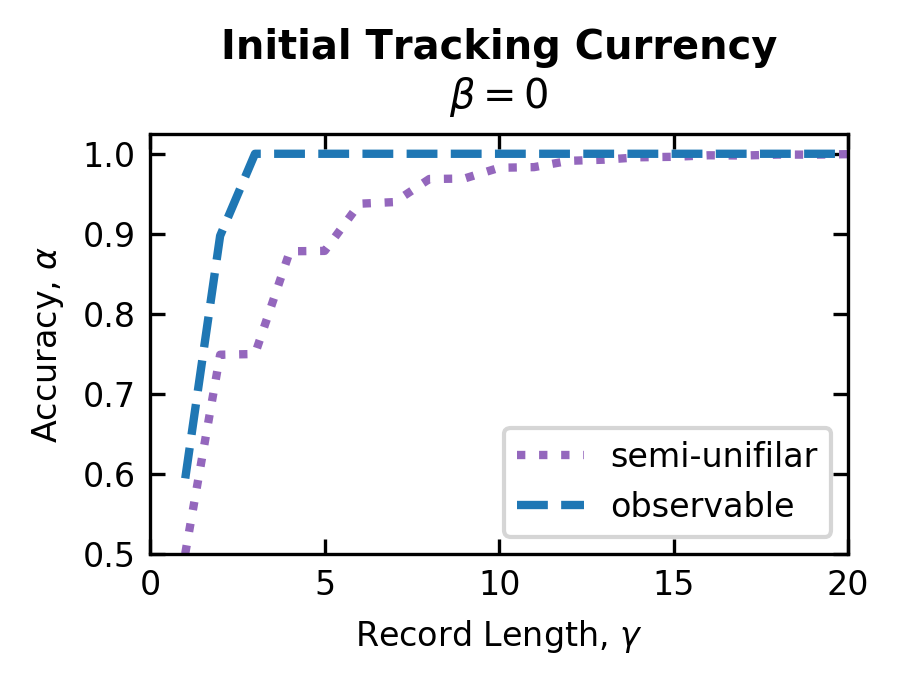}
	\caption{Real-time (lag $\beta=0$) tracking accuracy for the semi-unifilar (dotted purple) and observable (dashed blue) graphs. This is essentially a horizontal slice of the data shown in Figure~\ref{fig:burnInContour}. The semi-unifilar graph asymptotically approaches 100\% accuracy, while the observable graph obtains 100\% accuracy for $\gamma>2$.}
	\label{fig:burnInAlpha}
\end{figure}

Now consider the addition of an indicator node into the lower right (\greendiamond, \orangetriangle) cycle, as shown in Figure~\ref{sf:obsMit}.
The modified graph is observable.
Consistent with this, 100\% tracking accuracy is obtained after a fixed burn-in period of two time steps (see Figure~\ref{fig:burnInContour}b).
For the equal transition probabilities used here, this corresponds to approximately $5{\times}$ faster burn-in compared to the semi-unifilar case.

\section{Conclusions and Future Work}
\label{sec:conc}
This paper has presented a new pathology-based taxonomy of colored graph observability classes which unifies the results of \cite{ShengIEEE2005,CrespiACM2008,JungersDAM2011} into a common framework.
The three colored graph pathologies identified provide an intuitive picture of the differences between the various observability classes, and the expanded concept of $(\alpha,\beta,\gamma)$-currency provides a principled way of reasoning about the effects of the various pathologies.
Numerical experiments have shown the ability to change the observability class of a graph through the addition of indicator nodes, providing a more complete view of this topic than the initial experimental results in \cite{ChilenskiSPIE2018}.
Furthermore, we have shown that at least one form of the Indicator Node Selection Problem is in NP-Complete.
The formulation of the taxonomy has intentionally avoided questions of transition/emission probabilities and initial state distributions so that the results are as general as possible, and hence can be applied to any situation where a hidden state sequence is to be reconstructed from noisy/potentially ambiguous observations.

In terms of possible future research directions, consider the following observations.
There is a large well-known literature on problems of graph colorability.  Loosely speaking, those problems
ask how many colors are required to color nodes in a graph so that no two adjacent nodes have the same color
\cite{jensen2011graph}.  The classic problem in this area is of course the four color problem for planar graphs.

The ``chromatic number'' of a graph is the smallest number of colors that make the graph colorable in the above sense.  Determining the chromatic number is known to be in NP-complete, and therefore the most efficient algorithms currently known are of exponential complexity \cite{garey1976some}.  

By analogy, we can imagine defining a 
``{\em p}-observability number'' of a directed graph as being the smallest number of colors required to make the graph have the {\em p}-observability property, where such a {\em p}-observability property is one of the observability properties discussed in this paper.  Moreover, just as the chromatic polynomial, $P_G(k)$, is a polynomial with
the property that
the number of legal classical colorings of graph $G$ using $k$ colors is precisely $P_G(k)$ (so that
the chromatic number of $G$ is the smallest $k$ for which $P_G(k)>0$), we imagine there might be 
a ``{\em p}-observability polynomial'' with the similar property for the ``{\em p}-observability'' of a graph.
Such a ``{\em p}-observability polynomial'' would inform us about how difficult or easy
realizing the ``{\em p}-observability property'' would be for a specific graph.

To illustrate the intuition of how observability relates to chromatic numbers, consider the following
construct.  Given a directed uncolored graph $G=(V,E)$, create the auxiliary graph $A(G)=(V_A,E_A)$ for which
\begin{gather}
V_A = \{ (v_i,v_j) | v_i \neq v_j \}
\end{gather}
and
\begin{gather}
	E_A = \{ \big((v_i,v_j), (v_k,v_l)\big) | (v_i,v_k), (v_j,v_l) \in E\}.
\end{gather}
Note that $A$ is a supergraph of the auxiliary graph $G^2$ which is used to characterize partly \ap observability.
In particular, the cycles in $A$ are the non-intersecting closed paths of $G$ that are \emph{potentially} indistinguishable from each other.
Now select a node from every cycle, $c$, in $A$, which is essentially just a pair $(v_c,v'_c)$ of nodes from $G$.
Consider the graph $B = (V,E_B)$ with the same nodes as $G$ and with edges $(v_c,v'_c)$ 
as selected from the cycles of $A$.
$B$ is made undirected through inclusion of the edge $(v'_c, v_c)$ whenever the edge $(v_c, v'_c)$ has been included.

Observe that every traditional coloring of $B$ (i.e., no neighboring nodes have the same color) leads to a 
corresponding coloring of $A$ that has no same colored cycles because of how
we constructed $B$.  Thus the chromatic number of $B$ has a clear relationship
to the number of ways $G$ could be colored to be observable 
through such a selection of nodes from the cycles in $A$.
Note, however, that the construction of $B$ is not unique: there can be multiple ways of selecting nodes from the cycles in $A$ which lead to different chromatic numbers for $B$.
An example of this construction applied to the graph in Figure~\ref{fig:sepCyc} is given in Figure~\ref{fig:ABexample}.
\begin{figure}
	\centering
	\subfloat[Auxiliary graph $A$]{\includegraphics[scale=0.125]{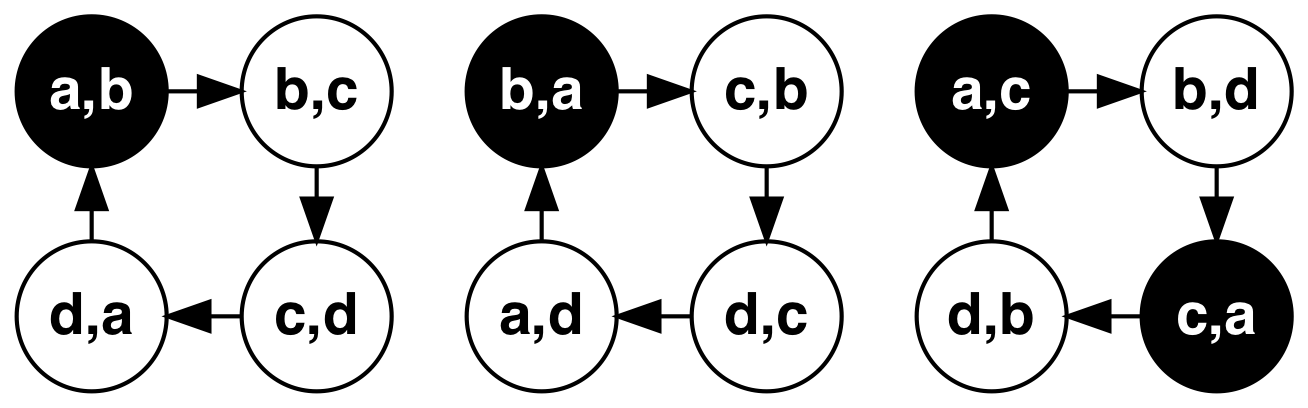}}%
	\quad\quad
	\subfloat[Auxiliary graph $B$]{\includegraphics[scale=0.125]{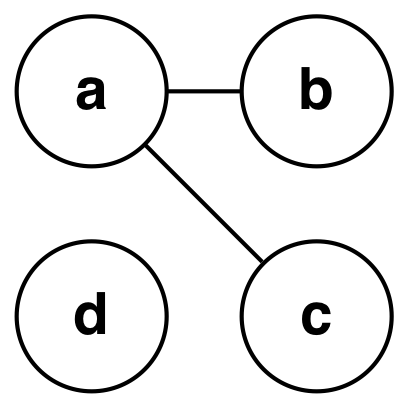}}%
	
	\caption{Auxiliary graphs $A$ and $B$ for the graph in Figure~\ref{fig:sepCyc}. Note that $A$ is not a colored graph, but the nodes of $A$ which were selected from each cycle to construct $B$ are shown in black here. $B$ is two-colorable, which shows that the original graph can be made partly \ap observable with no more than two colors.}
	\label{fig:ABexample}
\end{figure}

This by no means proves anything about any ``$p$-observability'' polynomial but
it does sketch out a concrete relationship between chromatic numbers and observability
that could provide insight into future work along these lines.

\section*{Acknowledgements}
Distribution Statement ``A'' (Approved for Public Release, Distribution Unlimited)

The research effort depicted was sponsored by the Air Force Research Laboratory (AFRL) and the Defense Advanced Research Projects Agency (DARPA) under the Leveraging the Analog Domain for Security (LADS) program under contract number FA8650-16-C-7622. In particular, we thank Dr.\ Angelos Keromytis and Mr.\ Ian Crone, the past and present DARPA program managers of LADS, for their encouragement and support throughout the program. 

This research was developed with funding from the Defense Advanced Research Projects Agency (DARPA).

The views, opinions and/or findings expressed are those of the author and should not be interpreted as representing the official views or policies of the Department of Defense or the U.S.\ Government.

\bibliographystyle{IEEEtran}
\bibliography{taxonomy}

\begin{IEEEbiographynophoto}{Mark Chilenski}
received the BS degree in aeronautical and astronautical engineering from the University of Washington in 2010 and the PhD degree in nuclear science and engineering from the Massachusetts Institute of Technology in 2016. He is a senior scientist at Systems \& Technology Research LLC. His research interests include machine learning, Bayesian inference, and cybersecurity.
\end{IEEEbiographynophoto}

\begin{IEEEbiographynophoto}{George Cybenko}
received his B.Sc. and Ph.D.
degrees in Mathematics from the University of Toronto and Princeton. He
is currently the Dorothy and Walter Gramm Professor
of Engineering at Dartmouth. His research interests include cyber
security, advanced machine learning algorithms and information
deception.
\end{IEEEbiographynophoto}

\begin{IEEEbiographynophoto}{Isaac Dekine}
received the BS and MS degrees in electrical and computer engineering from Carnegie Mellon University in 2006. He is a senior engineer at Systems \& Technology Research LLC. His research interests include RF system design, signal processing, and cybersecurity.
\end{IEEEbiographynophoto}

\begin{IEEEbiographynophoto}{Piyush Kumar}
received the Master of Science degree in Physics from the Indian Institute of Technology Kharagpur in 2001, MS degree in Physics from the University of Chicago in 2004 and the PhD degree in Physics from the University of Michigan Ann Arbor in 2007. He is a lead scientist at Systems \& Technology Research LLC. His research interests include applications of probabilistic methods to problems in physics and engineering, machine learning, and graph theory.
\end{IEEEbiographynophoto}

\begin{IEEEbiographynophoto}{Gil Raz}
received a bachelor's degree in electrical engineering from the Technion -- Israel Institute of Technology in 1988 and a PhD in electrical engineering (minor in mathematics) from the University of Wisconsin -- Madison in 1998. He is a chief scientist at Systems \& Technology Research LLC. His research interests include applied mathematics and statistics for solving problems in multiple application areas.
\end{IEEEbiographynophoto}

\clearpage

\section*{Appendix: Proof of Theorem~\ref{INSP-theorem}}
This appendix provides a proof of Theorem~\ref{INSP-theorem}.
Recall the definition of the Indicator Node Selection Problem from Section~\ref{sec:INSP}:
\begin{quote}
	{\em Indicator Node Selection Problem}: Given a node-colored directed graph, $G=(V,E,L,\Phi)$, and a subset of edges, $F \subseteq E$, can indicator nodes be added to some edges in $F$ so that the resulting graph is partly \ap observable?
\end{quote}

Using this definition, we can then state the following theorem:
\setcounter{theorem}{0}
\begin{theorem}
	The Indicator Node Selection Problem is in NP-Complete. \label{INSP-theorem-app}
\end{theorem}
\begin{proof}
The idea for the reduction used here was inspired by
earlier work showing that edge coloring a graph to make it
partly observable with a minimal number of colors is in NP-Complete \cite{JungersDAM2011}.

Our proof is also based on a reduction from the Monochromatic-Triangle Problem
which is known to be in NP-Complete \cite{garey2002computers} but uses a different reduction mapping.
The Monochromatic-Triangle Problem takes as input an undirected graph, $G=(V,E)$,
and asks whether there exists a two-coloring of the edges in $E$ so that no triangle in $G$
has all edges the same color (that is, no triangle is monochromatic).

We will develop the argument based on several steps.  Each step involves
a stage in the construction of a directed node-colorable graph $G'$ from $G$ with the
property that there is a positive answer to the {\em Indicator Node Selection Problem} for $G'$ if and only if
$G$ can be edge two-colored so that no triangle has all edges the same color.

{\em Step 1:}  Start by enumerating all triangles in $G$, say $T_1, T_2, ... , T_S$,
and creating nodes in $G'$ for each triangle.  Label these nodes by the triangle they
represent. Next, for each edge in $G$
create a unique pair of nodes in $G'$ and a directed edge between those nodes.
We call these edges ``real edges.''

Then for each triangle node add a directed edge to the tails of the
node-edge combinations created for each of its three edges.  
We call these edges ``connector'' edges.

A depiction of this construct is shown in Figure~\ref{fig:gcfig1}
for one triangle and Figure~\ref{fig:gcfig2} for two triangles.
These figures also show the insertion of an indicator node on one
of the ``real triangle edges.''  This indicator node is a surrogate for coloring
the corresponding edge in $G$.  Note that the underlying $G$ 
 can be edge two-colored so that no triangle has all edges the same color
 if and only if not all three real edges either have or do not have an
 indicator node inserted. The example constructs in Figures \ref{fig:gcfig1} and
\ref{fig:gcfig2} indicate that at least triangle $T_i$ has one edge colored
differently from the other two edges.

Note that this construct is polynomial in the size of $G$.

\begin{figure}[h]
	\centering
	\includegraphics[width=2.5in,keepaspectratio]{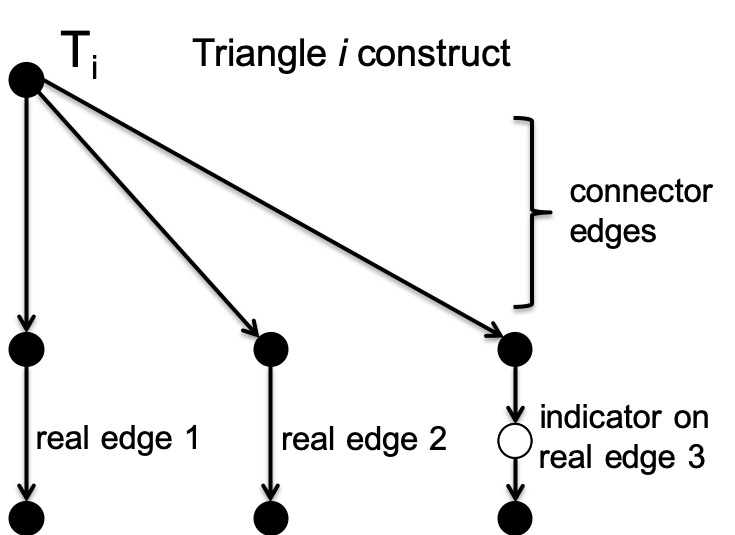}
	\caption{The construct for a single triangle}
	\label{fig:gcfig1}
\end{figure}

\begin{figure}[h]
\begin{center}
 \includegraphics[width=3in,keepaspectratio]{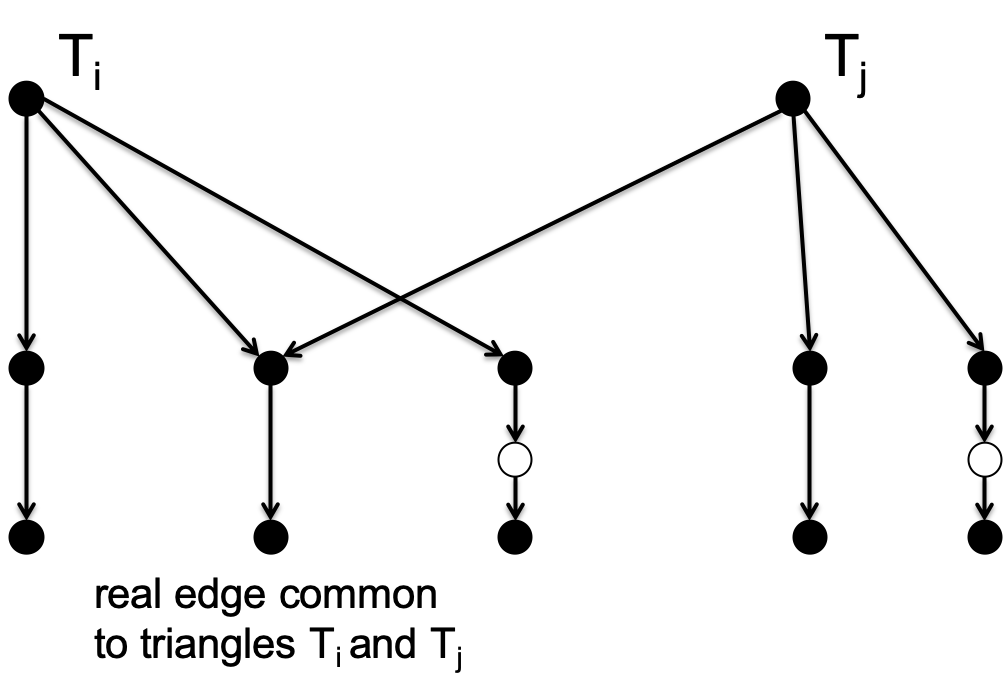}
 \caption{The construct for two triangles sharing an edge.}
 \label{fig:gcfig2}
 \end{center}
\end{figure}

{\em Step 2:}  For each triangle $T_i$,  create two
additional copies of that triangle node resulting in a total of three copies,
$T_{i1}, T_{i2}, T_{i3}$, with those copies all pointing to to the starts of the ``real edges''
in the underlying triangle.

\begin{figure}[h]
\begin{center}
 \includegraphics[width=3in,keepaspectratio]{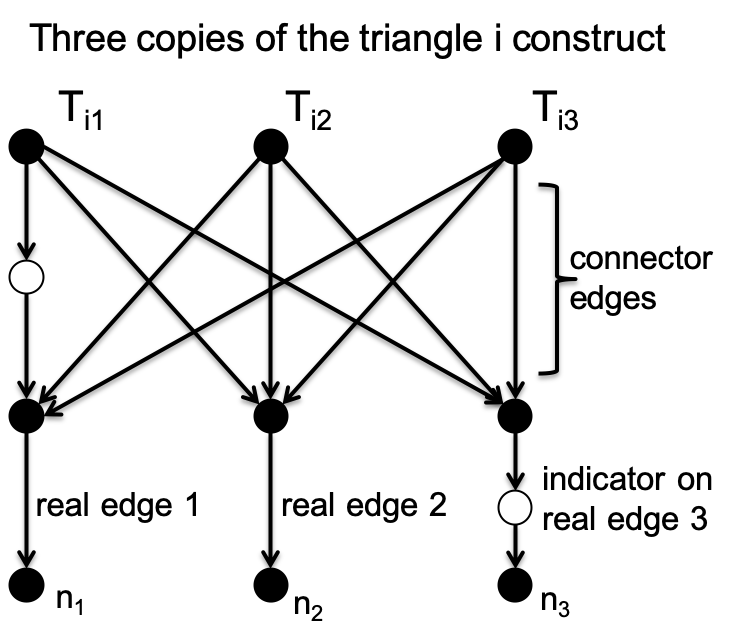}
 \caption{Repeat the nodes for a single triangle to create three copies.}
 \label{fig:gcfig5}
 \end{center}
\end{figure}

A depiction of this construct is shown in Figure~\ref{fig:gcfig5}.
Note that this part of the construct is still polynomial in the size of $G$.

The edges in this constructed graph together with the corresponding graphs
constructed for all other triangles in $G$ form the subset $F$
of the ultimate graph $G'$ which are candidates for insertion
of indicator nodes.  Figure~\ref{fig:gcfig5} shows two edges having
indicator nodes, as an example.

Now consider the paths between the three nodes $T_{i1}, T_{i2},T_{i3}$ corresponding to 
a single triangle in $G$ and the nodes at the bottom of
the real edges, labelled $n_1,n_2$ and $n_3$ in Figure~\ref{fig:gcfig5}, for triangle $T_i$.
Each path starts with a connector edge followed by a real edge.
These edges can have indicator nodes on them or not.

The fundamental observation that is key for this reduction
is that if all three real edges of $T_i$ are colored the same (that is, all have
indicator nodes or none have indicator nodes) then there are two disjoint paths
from two of the $T_{ij}$ to two of the nodes  $n_1,n_2$ and $n_3$, for any combination of
indicator nodes on the connector edges, that have the same combination of indicator nodes.

Said another way, if all real edges either have or do not have indicator nodes,
then there are two indistinguishable paths from two distinct triangle copy nodes, $T_{ij}$,
to two distinct $n_1,n_2,n_3$.  To see this, consider the three paths from 
each $T_{ij}$ to $n_j$ going straight down in Figure~\ref{fig:gcfig5}.  Because
all real edges are the same in terms of indicator nodes, there are only two possibilities
for the connector edges but there are three paths.  So at least two of these paths
must have the same arrangement of indicator nodes.

Conversely, if we have a triangle in which at least one, but not all, real edges have
an indicator node, then we can assign indicator nodes on the connector edges so
that a unique $n_i$ can be identified by any path from any of the $T_{ij}$ down.
For example, Figure~\ref{fig:gcfig6} shows such an assignment.

\begin{figure}[h]
\begin{center}
 \includegraphics[width=3in,keepaspectratio]{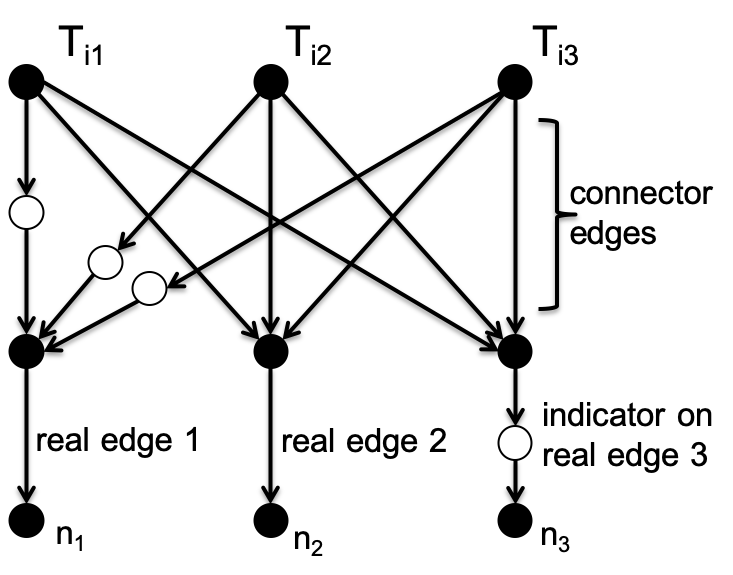}
 \caption{If not all real edges in a triangle are the same with respect to having indicator nodes or not,
 then the connector edges can be assigned indicator nodes to make the terminal
 nodes $n_k$ uniquely identifiable given any path from one of the $T_{ij}$ to $n_k$.}
  \label{fig:gcfig6}
 \end{center}
\end{figure}

{\em Step 3:}  Now that we have constructed the subgraph $F$ which arises in the
Indicator Node Selection Problem, we need to construct the rest
of the graph $G'$ so that indicator nodes can be assigned to $F$
to make the overall graph $G'$ partly observable if and only if the edges in the
underlying  $G$ can be two-colored so that no triangle in $G$ has all
edges the same color.

Recall that $G$ has $S$ triangles.  Construct a binary tree that has
$S$ leaves with a node coloring such that a left node child is colored black and
a right node child is colored white and the paths from root to leaves are all
the same length.  This can be done simply by creating a binary tree
with $2^{\lceil \log_2 S \rceil}$ leaves and using only $S$ of the leaves.

Now make a one-to-one assignment of the $S$ leaves in the tree to the
$S$ triangles in $G$.  A node color sequence from the root of the tree
to a leaf uniquely identifies a triangle in $G$.  Create three copies of this
tree and connect leaves to the three copies of the triangle nodes.

Note that a node color sequence from the root of any tree down to the
triangle nodes $T_{ij}$ uniquely identifies $i$ (which triangle is at
the leaf) but not $j$ (which copy of that triangle).

\begin{figure*}[h]
\begin{center}
 \includegraphics[width=6in,keepaspectratio]{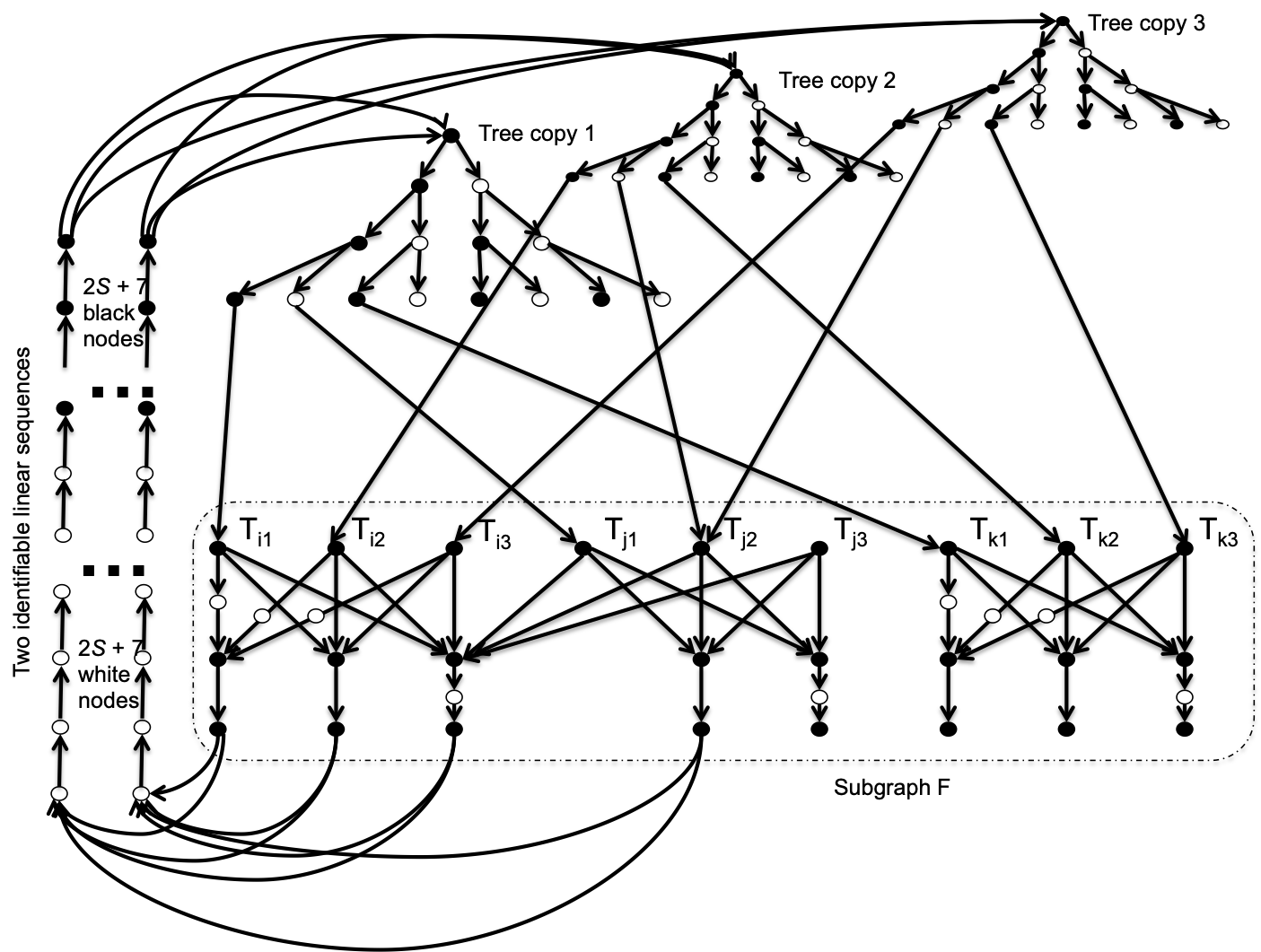}
 \caption{The final construction of $G'$ with the subgraph $F$ for insertion of
 indicator nodes enclosed in the dashed rectangle.  }
 \label{fig:gcfig4}
 \end{center}
\end{figure*}

{\em Step 4:}  The final step involves connecting the bottom nodes $n$
in $F$ to the roots of the three trees.  For this, create two separate linear
arrays starting with $2S+7$ white nodes followed by $2S+7$ black nodes.
Connect the bottom nodes to start of each of the linear arrays and the end
nodes of the linear arrays to each of the three tree copies.  This construct is
depicted in Figure~\ref{fig:gcfig4}.

Now consider a path through this constructed $G'$ and observe the colors
which are either black or white along the path.  When we see a sequence of
$2S+7$ whites followed by $2S+7$ blacks, we know we just traversed
one of the linear arrays on the left of Figure~\ref{fig:gcfig4} because
that is the only place where such a color sequence can occur.

After the  $2S+7$ blacks, we enter one of the three binary trees but we
do not know which one of the trees.  We traverse the tree and because we
know there are ${\lceil \log_2 S \rceil}$ edges on all paths in all trees,
we know when we reach one of the triangle nodes and which triangle
because of the binary encoding but we do not know which one
of the three triangle nodes  $T_{i1}, T_{i2},T_{i3}$ we have reached because
we do not know which copy of the tree we traversed.

Recall that we constructed $F$ in such a way that if all real edges
of a triangle either all have indicator nodes or none have indicator nodes, then
there are must be two disjoint paths from two of the $T_{ij}$ to two of the
$n_j$ nodes for that triangle that are identically colored.  We close
those disjoint paths by passing to two different linear arrays and then
passing from the linear arrays to the two trees corresponding to the
two triangle node copies that started the disjoint paths identified above.

This creates two separated cycles with the same coloring and shows that any
assignment of indicator nodes to edges within $F$ that results in
all three real edges of any triangle being the same means that the
resulting $G'$ is \emph{not} partly \ap observable.

Moreover, an assignment of indicator nodes to $F$ for which
all real edges of a triangle in $G$ have indicator nodes or not
means that all edges of that triangle in $G$ are colored the same.

On the other hand, suppose that the constructed graph $G'$
has an assignment of indicator nodes to the edges in $F$ that
makes $G'$ partly \ap observable.  This means that all same colored
paths must eventually intersect which in turn implies that 
not all three real edges corresponding to any triangle in $F$ can all have indicator nodes
or not.  That implies that the original graph $G$ can have edges
colored so that no triangle has all edges the same color.

This completes the proof of Theorem~\ref{INSP-theorem}.
\end{proof}

\end{document}